\documentclass[11pt]{article}

\usepackage{fullpage,times,url,bm}

\usepackage{amsthm,amsfonts,amsmath,amssymb,epsfig,color,float,graphicx,verbatim}
\usepackage{algorithm,algorithmic}
\usepackage{bbm}
\usepackage{natbib}

\usepackage{hyperref}
\hypersetup{
	colorlinks   = true, 
	urlcolor     = blue, 
	linkcolor    = blue, 
	citecolor   = black 
}

\newtheorem{theorem}{Theorem}[section]
\newtheorem{proposition}{Proposition}[section]
\newtheorem{lemma}{Lemma}[section]

\newtheorem{remark}{Remark}[section]

\usepackage{amssymb}
\newcommand{\divides}{\mid}

\usepackage{dsfont}
\newcommand{\onefunc}{\mathds{1}}

\newcommand{\stam}[1]{}

\newcommand{\bx}{\mathbf{x}}
\newcommand{\bw}{\mathbf{w}}

\newcommand{\bb}{\mathbf{b}}
\newcommand{\bu}{\mathbf{u}}

\newcommand{\bz}{\mathbf{z}}
\newcommand{\bc}{\mathbf{c}}

\newcommand{\by}{\mathbf{y}}

\newcommand{\co}{{\cal O}}

\newcommand{\cd}{{\cal D}}

\newcommand{\cc}{{\cal C}}

\newcommand{\ci}{{\cal I}}

\newcommand{\cn}{{\cal N}}

\DeclareMathOperator*{\sign}{sign}

\newcommand{\bbs}{{\mathbb S}}
\newcommand{\reals}{{\mathbb R}}

\newcommand{\integers}{{\mathbb Z}}

\newcommand{\zero}{{\mathbf{0}}}

\DeclareMathOperator{\poly}{poly}
\DeclareMathOperator{\polylog}{polylog}
\DeclareMathOperator*{\E}{\mathbb{E}}
\DeclareMathOperator{\msb}{MSB}
\DeclareMathOperator{\bin}{bin}
\DeclareMathOperator{\betadist}{Beta}

\newcommand{\inner}[1]{\langle #1 \rangle}
\newcommand{\norm}[1]{\left\|#1\right\|}
\newcommand{\snorm}[1]{\|#1\|} 
\newcommand{\tbx}{{\tilde{\bx}}}
\newcommand{\tx}{{\tilde{x}}}

\title{Neural Networks with Small Weights and \\
Depth-Separation Barriers}
\author{Gal Vardi\qquad Ohad Shamir\\
Weizmann Institute of Science\\
\texttt{\{gal.vardi,ohad.shamir\}@weizmann.ac.il}
}
\date{}
\begin{document}

\maketitle

\begin{abstract}
In studying the expressiveness of neural networks, an important question is whether there are functions which can only be approximated by sufficiently deep networks, assuming their size is bounded. However, for constant depths, existing results are limited to depths $2$ and $3$, and achieving results for higher depths has been an important open question. In this paper, we focus on feedforward ReLU networks, and prove fundamental barriers to proving such results beyond depth $4$, by reduction to open problems and natural-proof barriers in circuit complexity. To show this, we study a seemingly unrelated problem of independent interest: Namely, whether there are polynomially-bounded functions which require super-polynomial weights in order to approximate with constant-depth neural networks. We provide a negative and constructive answer to that question, by showing that if a function can be approximated by a polynomially-sized, constant depth $k$ network with arbitrarily large weights, it can also be approximated by a polynomially-sized, depth $3k+3$ network, whose weights are polynomially bounded.
\end{abstract}

\section{Introduction}
\label{sec:intro}

The {\em expressive power} of feedforward neural networks has been extensively studied in recent years. It is well-known that sufficiently large depth-$2$ neural networks, using reasonable activation functions, can approximate any continuous function on a bounded domain (\cite{cybenko1989approximation,funahashi1989approximate,hornik1991approximation,barron1994approximation}). However, the required size of such networks can be exponential in the input dimension, which renders them impractical. From a learning perspective, both theoretically and in practice, the main interest is in neural networks whose size is at most polynomial in the input dimension.

When considering the expressive power of neural networks of bounded size, a key question is what are the tradeoffs between the width and the depth.
Overwhelming empirical evidence indicates that deeper networks tend to perform better than shallow ones, a phenomenon supported by the intuition that depth, providing compositional expressibility, is necessary for efficiently representing some functions.
From the theoretical viewpoint, quite a few works in the past few years have explored the beneficial effect of depth on increasing the expressiveness of neural networks.
A main focus is on {\em depth separation}, namely, showing that there is a function $f:\reals^d \rightarrow \reals$ that can be approximated by a $\poly(d)$-sized network of a given depth, with respect to some input distribution, but cannot be approximated by $\poly(d)$-sized networks of a smaller depth.
Depth separation between depth $2$ and $3$ was shown by \cite{eldan2016power} and \cite{daniely2017depth}. However, despite much effort, no such separation result is known for any constant greater than $2$. Thus, it is an open problem whether there is separation between depth $3$ and some constant depth greater than $3$.
Separation between networks of a constant depth and networks with $\poly(d)$ depth was shown by \cite{telgarsky2016benefits} (see related work section below for more details).

In fact, a similar question has been extensively studied by the theoretical computer science community over the past decades, in the context of Boolean and threshold circuits of bounded size.
Showing limitations for the expressiveness of such circuits (i.e. {\em circuit lower bounds})
can contribute to our understanding of the $P\neq NP$ question, and can have other significant theoretical implications (\cite{arora2009computational}).
Despite many attempts, the results on circuit lower bounds were limited.
In a seminal work, \cite{razborov1997natural} described a main technical limitation of current approaches for proving circuit lower bounds: They defined a notion of ``natural proofs" for a circuit lower bound (which include current proof techniques), and showed that obtaining lower bounds with such proof techniques would violate a widely accepted conjecture, namely,
that pseudorandom functions exist. This {\em natural-proof barrier} explains the lack of progress on circuit lower bounds.
More formally, they show that if a class $\cc$ of circuits contains a family of pseudorandom functions, then showing for some function $f$ that $f \not \in \cc$ cannot be done with a natural proof. As a result, if we consider the class $\cc$ of $\poly(d)$-sized circuits of some bounded depth $k$, where $k$ is large enough so that $\cc$ contains a pseudorandom function family, then it will be difficult to show that some functions are not in $\cc$, and hence that these functions require depth larger than $k$ to express.

An object closer to actual neural networks are threshold circuits. These are essentially neural networks with a threshold activation function in all neurons (including the output neuron), and where the inputs are in $\{0,1\}^d$.
The problem of depth separation in threshold circuits was widely studied (\cite{razborov1992small}). This problem requires, for some integer $k$, a function that cannot be computed by a threshold circuit of width $\poly(d)$ and depth $k$, but can be computed\footnote{Note that in this literature it is customary to require exact representation of the function, rather than merely approximating it.} by a threshold circuit of width $\poly(d)$ and depth $k'>k$. \cite{naor2004number} and \cite{krause2001pseudorandom} showed a candidate pseudorandom function family computable by threshold circuits of depth $4$, width $\poly(d)$, and $\poly(d)$-bounded weights. By \cite{razborov1997natural}, it implies that for every $k' > k \geq 4$, there is a natural-proof barrier for showing depth separation between threshold circuits of depth $k$ and depth $k'$. As for smaller depths, a separation between threshold circuits of depth $3$ and some $k>3$ is a longstanding open problem (although there is no known natural-proof barrier in this case), and separation between threshold circuits of depth $2$ and $3$ is known under the assumption that the weight magnitudes are $\poly(d)$ bounded (\cite{hajnal1987threshold}).

Since a threshold circuit is a special case of a neural network with threshold activation and where the inputs and output are Boolean, it is natural to ask whether the barriers to depth separation in threshold circuits have implications on the problem of depth separation in neural networks. Such implications are not obvious, since neural networks have real-valued inputs and outputs (not necessarily just Boolean ones), and a continuous activation function. Thus, it might be possible to come up with a depth-separation result, which crucially utilizes some function and inputs in Euclidean space. In fact, this can already be seen in existing results: For example, separation between threshold circuits of constant depth (TC$^0$) and threshold circuits of $\poly(d)$ depth (which equals the complexity class P/poly) is not known, but \cite{telgarsky2016benefits} showed such a result for neural networks. His construction is based on the observation that for one dimensional data, a network of depth $k$ is able to express a sawtooth function on the interval $[0,1]$ which oscillates $\co(2^k)$ times. Clearly, this utilizes the continuous structure of the domain, in a way that is not possible with Boolean inputs. Also, the depth-$2$ vs.~$3$ separation results of \cite{eldan2016power} and \cite{daniely2017depth} rely on harmonic analysis of real functions. Finally, the result of \cite{eldan2016power} does not make any assumption on the weight magnitudes, whereas relaxing this assumption for the parallel result on threshold circuits is a longstanding open problem (\cite{razborov1992small}).

\subsection*{Main Result 1: Barriers to Depth Separation}

In this work, we focus on real-valued neural networks with the ReLU activation function, and show (under some mild assumptions on the input distribution and on the function) that any depth-separation result between neural networks of depth $k \geq 4$ and some constant $k'>k$ would imply a depth separation result between threshold circuits of depth $k-2$ and some constant greater than $k-2$. Hence, showing depth separation with $k=5$ would solve the longstanding open problem of separating between threshold circuits of depth $3$ and some constant greater than $3$. Showing depth separation with $k \geq 6$ would solve the open problem of separating between threshold circuits of depth $k-2$ and some constant depth greater than $k-2$, which is especially challenging due to the natural-proof barrier for threshold circuits of depth at least $4$. Finally, showing depth separation with $k=4$ 
would solve the longstanding open problem of separating between threshold circuits of depth $2$ (with arbitrarily large weights) and some constant greater than $2$ (we note that separation between threshold circuits of depth~$2$ and~$3$ is known only under the assumption that the weight magnitudes are $\poly(d)$ bounded).
The result applies to both continuous and discrete input distributions.
Thus, we show a barrier to depth separation, that explains the lack of progress on depth separation for constant-depth neural networks of depth at least $4$.

While this is a strong barrier to depth separation in neural networks, it should not discourage researchers from continuing to investigate the problem. First, our results focus on plain feedforward ReLU networks, and do not necessarily apply to other architectures. Second, we do make some assumptions on the input distribution and the function, which are mild but perhaps can be circumvented (or alternatively, relaxed). Third, our barrier does not apply to separation between depth $3$ and some larger constant.
That being said, we do show that in order to achieve separation between depth $k \geq 3$ and some constant $k'>k$, some different approach than these used in current results would be required.
As far as we know, in all existing depth-separation results for continuous input distributions (e.g., \cite{eldan2016power,daniely2017depth,telgarsky2016benefits,safran2017depth,liang2016deep,yarotsky2017error,safran2019depth})
the functions are either of the form $f(\bx)=g(\snorm{\bx})$ or of the form $f(\bx)=g(x_1)$ for some $g:\reals \rightarrow \reals$. Namely, $f$ is either a radial function, or a function that depends only on one component\footnote{In \cite{daniely2017depth} the function is not radial, but, as shown in \cite{safran2019depth}, it can be reduced to a radial one.}.
We show that for functions of these forms, networks of a constant depth greater than $3$ do not have more power than networks of depth $3$.

\subsection*{Main Result 2: Effect of Weight Magnitude on Expressiveness}

To establish our depth-separation results, we actually go through a seemingly unrelated problem
of independent interest: Namely, what is the impact on expressiveness if we force the network weights to have reasonably bounded weights (say, $\poly(d)$). This is a natural restriction: Exponentially-large weights are unwieldy, and moreover, most neural networks used in practice have small weights, due to several reasons related to the training process, such as regularization, standard initialization of the weights to small values, normalization heuristics, and techniques to avoid the exploding gradient problem (\cite{goodfellow2016deep}). Therefore, it is natural to ask how bounding the size of the weights affects the expressive power of neural networks. As far as we know, there are surprisingly few works on this, and current works on the expressiveness of neural networks often assume that the weights may be arbitrarily large, although this is not the case in practice.

If we allow arbitrary functions, there are trivial cases where limiting the weight magnitudes hurts expressiveness. For example, consider the function $f:[0,1]^d \rightarrow \reals$, where for every $\bx=(x_1,\ldots,x_d)$ we have $f(\bx)=x_1 \cdot 2^d$. Clearly, $f$ can be expressed by a neural network of depth $1$ with exponential (in $d$) weights. This function cannot be approximated with respect to the uniform distribution on $[0,1]^d$ by a network of constant depth with $\poly(d)$ width and $\poly(d)$-bounded weights, since such a network cannot compute exponentially-large values. However, functions of practical interest only have constant or $\poly(d)$-sized values (or at least can be well-approximated by such functions). Thus, a more interesting question is whether for approximating such bounded functions, we may ever need weights whose size is more than $\poly(d)$.

In our paper, we provide a negative answer to this question, in the following sense: Under some mild assumptions on the input distribution, if the function can be approximated by a network with ReLU activation, width $\poly(d)$, constant depth $k$ and arbitrarily large weights, then we show how it can be approximated by a network with ReLU activation, width $\poly(d)$, depth $3k+3$, and with weights whose absolute values are bounded by some $\poly(d)$ or by a constant. The result applies to both continuous and discrete input distributions.

The two problems that we consider, namely depth-separation and the power of small weights, may seem unrelated. Indeed, each problem considers a different aspect of expressiveness in neural networks. However, perhaps surprisingly, the proofs for our results on barriers to depth separation follow from our construction of networks with small weights. In a nutshell, the idea is that our deeper small-weight network is such that most layers implement a threshold circuit. Thus, if we came up with a ``hard'' function $f$ that provably requires much depth to express with a neural network, then the threshold circuit used in expressing it (via our small-weight construction) also provably requires much depth -- since otherwise, we could make our small-weight network shallower, violating the assumption on $f$. This would lead to threshold-circuit lower bounds. See Section \ref{sec:proof ideas} for more details on the proof ideas.

\subsection*{Related Work}

\textbf{Depth separation in neural networks.}
As we already mentioned, depth separation between depth $2$ and $3$ was shown by \cite{eldan2016power} and \cite{daniely2017depth}.
In \cite{eldan2016power} there is no restriction on the weight magnitudes of the depth-$2$ network, while \cite{daniely2017depth} assumes that the weights are bounded by $2^d$.
The input distributions there are continuous.
A separation result between depth~$2$ and~$3$ for discrete inputs is implied by \cite{martens2013representational}, for the function that computes inner-product mod $2$ on binary vectors (see also a discussion in \cite{eldan2016power}).

In \cite{telgarsky2016benefits}, it is shown that there exists a family of univariate functions $\{\varphi_k\}_{k=1}^\infty$ on the interval $[0,1]$, such that for every $k$ we have:
\begin{itemize}
\item The function $\varphi_k$ can be expressed by a network of depth $k$ and width $\co(1)$.
\item The function $\varphi_k$ cannot be approximated by any $o(k/\log(k))$-depth, $\poly(k)$-width network with respect to the uniform distribution on $[0,1]$.
\end{itemize}
To rewrite this as a depth separation result in terms of a dimension $d$, consider the functions $\{f_d\}_{d=1}^\infty$ where $f_d:[0,1]^d \rightarrow \reals$ is such that $f_d(\bx) = \varphi_d(x_1)$.
The result of \cite{telgarsky2016benefits} implies that the function $f_d$ can be expressed by a network of width $\co(1)$ and depth $d$, but cannot be approximated by a network of width $\poly(d)$ and constant depth.
Hence, there is separation between constant and polynomial depths. However, this result does not have implications for the problem of depth separation between constant depths.

In \cite{safran2017depth,liang2016deep,yarotsky2017error} another notion of depth separation is considered. They show that there are functions that can be $\epsilon$-approximated by a network of $\polylog(1/\epsilon)$ width and depth, but cannot be $\epsilon$-approximated by a network of $\co(1)$ depth unless its width is $\poly(1/\epsilon)$. Their results are based on a univariate construction similar to the one in \cite{telgarsky2016benefits}.

\textbf{Expressive power of neural networks with small weights.}
\cite{maass1997bounds} considered a neural network $N$ with a piecewise linear activation function in all hidden neurons, and threshold activation in the output neuron. Namely, $N$ computes a Boolean function.
He showed that if every hidden neuron in $N$ has fan-out $1$, and the $d$-dimensional input is from a certain discrete set, then there is a network $N'$ of the same size and same activation functions, that computes the same function, and its weights and biases can be represented by $\poly(d)$ bits. Thus, the weights in $N'$ are bounded by $2^{\poly(d)}$.
From his result, it is not hard to show the following corollary: Let $N$ be a network with ReLU activation in all hidden neurons and threshold activation in the output neuron, and assume that the input to $N$ is from $\{0,1\}^d$, and that $N$ has width $\poly(d)$ and constant depth. Then, there is a threshold circuit of $\poly(d)$ width, constant depth, and $\poly(d)$-bounded weights, that computes the same function. Note that this result considers exact computation of functions with binary inputs and output, while we consider approximation of functions with real inputs and output. 

Expressiveness with small weights was also studied in the context of threshold circuits. In particular, it is known that every function computed by a polynomial-size threshold circuit of depth $k$ can be computed by a polynomial-size threshold circuit of depth $k+1$ with weights whose absolute values are bounded by a polynomial or a constant (\cite{goldmann1992majority,goldmann1998simulating,siu1992neural}). This result relies on the fact that threshold circuits compute Boolean functions and does not apply to real-valued neural networks.

In the {\em weight normalization} method (\cite{salimans2016weight}), the weights are kept normalized during the training of the network. That is, all weight vectors of neurons in the network have the same Euclidean norm. Some approximation properties of such networks were studied in \cite{xu2018understanding}. The dependence of the sample complexity of neural networks on the norms of its weight matrices was studied in several recent works, e.g., \cite{bartlett2017spectrally,golowich2017size,neyshabur2017exploring}.

Our paper is structured as follows: In Section \ref{sec:preliminaries} we provide necessary notations and definitions, followed by our main results in Section \ref{sec:results}. We informally sketch our proof ideas in Section \ref{sec:proof ideas}, with all formal proofs deferred to Section \ref{sec:proofs}.

\section{Preliminaries}\label{sec:preliminaries}

\textbf{Notations.}
We use bold-faced letters to denote vectors, e.g., $\bx=(x_1,\ldots,x_d)$. For $\bx \in \reals^d$ we denote by $\snorm{\bx}$ the Euclidean norm. For a function $f:\reals^d \rightarrow \reals$ and a distribution $\cd$ on $\reals^d$, either continuous or discrete, we denote by $\snorm{f}_{L_2(\cd)}$ the $L_2$ norm weighted by $\cd$, namely $\snorm{f}_{L_2(\cd)}^2 = \E_{\bx \sim \cd}(f(\bx))^2$. Given two functions $f,g$ and real numbers $\alpha,\beta$, we let $\alpha f + \beta g$ be shorthand for $\bx \mapsto \alpha f(\bx) + \beta g(\bx)$. For a set $A$ we let $\onefunc_A$ denote the indicator function. For an integer $d \geq 1$ we denote $[d]=\{1,\ldots,d\}$. We use $\poly(d)$ as a shorthand for ``some polynomial in $d$".

\textbf{Neural networks.}
We consider feedforward neural	networks, computing functions from $\reals^d$ to $\reals$. The network is composed of layers of neurons, where each neuron computes a function of the form $\bx \mapsto \sigma(\bw^{\top}\bx+b)$, where $\bw$ is a weight vector, $b$ is a bias term and $\sigma:\reals\mapsto\reals$ is a non-linear activation function. In this work we focus on the ReLU activation function, namely, $\sigma(z) = [z]_+ = \max\{0,z\}$. For a matrix $W = (\bw_1,\ldots,\bw_n)$, we let $\sigma(W^\top \bx+\bb)$ be a shorthand for $\left(\sigma(\bw_1^{\top}\bx+b_1),\ldots,\sigma(\bw_n^{\top}\bx+b_n)\right)$, and define a layer of $n$ neurons as $\bx \mapsto \sigma(W^\top \bx+\bb)$. By denoting the output of the $i$-th layer as $O_i$, we can define a network of arbitrary depth recursively by $O_{i+1}=\sigma(W_{i+1}^\top O_i+\bb_{i+1})$, where $W_i,\bb_i$ represent the matrix of weights and bias of the $i$-th layer, respectively. The {\em weights vector} of the $j$-th neuron in the $i$-th layer is the $j$-th column of $W_i$, and its {\em outgoing-weights vector} is the $j$-th row of $W_{i+1}$. The {\em fan-in} of a neuron is the number of non-zero entries in its weights vector, and the {\em fan-out} is the number of non-zero entries in its outgoing-weights vector. Following a standard convention for multi-layer networks, the final layer $h$ is a purely linear function with no bias, i.e. $O_h=W_h^\top \cdot O_{h-1}$. We define the \emph{depth} of the network as the number of layers $l$, and denote the number of neurons $n_i$ in the $i$-th layer as the {\em size} of the layer. We define the {\em width} of a network as $\max_{i\in [l]}n_i$. We sometimes consider neural networks with multiple outputs. We say that a neural network has {\em $\poly(d)$-bounded weights} if for all individual weights $w$ and biases $b$, the absolute values $|w|$ and $|b|$ are bounded by some $\poly(d)$.

\textbf{Threshold circuits.}
A threshold circuit is a neural network with the following restrictions:
\begin{itemize}
\item The activation function in all neurons is $\sigma(z) = \sign(z)$. We define $\sign(z)=0$ for $z \leq 0$, and $\sign(z)=1$ for $z > 0$. A neuron in a threshold circuit is called a {\em threshold gate}.
\item The output gates also have a $\sign$ activation function. Hence, the output is binary.
\item We always assume that the input to a threshold circuit is a binary vector $\bx \in \{0,1\}^d$.
\item Since every threshold circuit with real weights can be expressed by a threshold circuit of the same size with integer weights (c.f. \cite{goldmann1998simulating}), we assume w.l.o.g. that all weights are integers.
\end{itemize}

\textbf{Probability densities.}
Let $\mu$ be the density function of a continuous distribution on $\reals^d$.
For $i \in [d]$ we denote by $\mu_i$ and $\mu_{[d] \setminus i}$ the marginal densities for $x_i$ and $\{x_1,\ldots,x_{i-1},x_{i+1},\ldots,x_d\}$ respectively. We denote by $\mu_{i | [d] \setminus i}$ the conditional density of $x_i$ given $\{x_1,\ldots,x_{i-1},x_{i+1},\ldots,x_d\}$. Thus, for every $i$ and $\bx^i=(x_1,\ldots,x_{i-1},x_{i+1},\ldots,x_d)$ we have $\mu(\bx) = \mu_{[d] \setminus i}(\bx^i) \mu_{i | [d] \setminus i}(x_i | \bx^i)$.

We say that $\mu$ has an {\em almost-bounded support} if for every $\delta = \frac{1}{\poly(d)}$ there is $R = \poly(d)$ such that $Pr_{\bx \sim \mu}(\bx \not \in [-R,R]^d) \leq \delta$.

We say that $\mu$ has an {\em almost-bounded conditional density} if for every $\epsilon = \frac{1}{\poly(d)}$ there is $M = \poly(d)$ such that for every $i \in [d]$ we have
\[
Pr_{\bx \sim \mu}\left(\sup_{t \in \reals} \mu_{i | [d] \setminus i}(t | x_1,\ldots,x_{i-1},x_{i+1},\ldots,x_d) > M\right)   \leq \epsilon~.
\]

\begin{remark}
In our results on continuous distributions we assume that the density $\mu$ has an almost-bounded support and an almost-bounded conditional density. While the first assumption is intuitive, the second is less standard. However, it is mild and intended to exclude distributions which are both continuous and with significant mass on extremely small domains. In Appendix~\ref{app:almost-bounded conditional density} we show that it holds, for example, for Gaussians (as long as the variance is at least $1/\poly(d)$ in all directions), mixtures of Gaussians, any distribution after a Gaussian smoothing, the uniform distribution on a ball, as well as distributions from existing depth-separation results.
In addition, with a slightly different proof, we also provide similar results for discrete distributions.
\end{remark}

\textbf{Functions approximation.}
\stam{
For a function $f:\reals^d \rightarrow \reals$ and a distribution $\cd$, either continuous or discrete, we denote
\[
\snorm{f}_{L_2(\cd)}^2 = \E_{\bx \sim \cd}(f(\bx))^2~.
\]
}
For $y \in \reals$ and $B>0$ we denote $[y]_{[-B,B]}=\max(-B, \min(y,B))$, namely, clipping $y$ to the interval $[-B,B]$.
We say that $f$ is {\em approximately $\poly(d)$-bounded} if for every $\epsilon=\frac{1}{\poly(d)}$ there is $B=\poly(d)$ such that
\[
\E_{\bx \sim \cd}\left(f(\bx)-[f(\bx)]_{[-B,B]}\right)^2 \leq \epsilon~.
\]
Note that if $f$ is bounded by some $B=\poly(d)$ then it is also approximately $\poly(d)$-bounded.

We say that $f$ can be {\em approximated by a neural network of depth $k$} (with respect to a distribution $\cd$) if for every $\epsilon=\frac{1}{\poly(d)}$ we have $\E_{\bx \sim \cd}(f(\bx)-N(\bx))^2 \leq \epsilon$ for some depth-$k$ network $N$ of width $\poly(d)$.

\textbf{Depth separation.}
We say that there is depth-separation between networks of depth $k$ and depth $k'$ for some integers $k'>k$, if there is a distribution $\cd$ on $\reals^d$ and a function $f:\reals^d \rightarrow \reals$ that can be approximated (with respect to $\cd$) by a neural network of depth $k'$ but cannot be approximated by a network of depth $k$.

We note that our definition of depth-separation is a bit weaker than most existing depth-separation results, which actually show difficulty of approximation even up to constant accuracy (and not just $1/\poly(d)$ accuracy). However, depth separation in that sense implies depth separation in our sense. Hence, the barriers we show here for depth separation imply similar barriers under this other (or any stronger) notion of depth separation.

\section{Results}\label{sec:results}

We start by presenting our results on small-weight networks, implying that extremely large weights do not significantly help neural networks to express approximately $\poly(d)$-bounded functions. We show this via a positive result: If an approximately $\poly(d)$-bounded function can be approximated by a network of constant depth $k$ and arbitrary weights, then it can also be approximated by a depth-$(3k+3)$ network with $\poly(d)$-bounded weights. The proof is constructive and explicitly shows how to convert one network to the other. We then proceed to use the proof construction, in order to establish depth-separation barriers for neural networks.

\subsection{Neural networks with small weights}
\label{sec:small weights}

We start with the case where the input distribution is continuous:

\begin{theorem}
\label{thm:poly weights}
Let $\mu$ be a density function on $\reals^d$ with an almost-bounded support and almost-bounded conditional density.
Let $f:\reals^d \rightarrow \reals$ be an approximately $\poly(d)$-bounded function, and let $k$ be a constant, namely, independent of $d$.
If $f$ can be approximated by a neural network of depth $k$ and width $\poly(d)$, then it can also be approximated by a neural network of depth $3k+3$, width $\poly(d)$, and $\poly(d)$-bounded weights.
\end{theorem}

We now show a similar result for the case where the input distribution is discrete:

\begin{theorem}
\label{thm:poly weights discrete}
Let $R(d)$ and $p(d)$ be any polynomials in $d$, and let $\ci = \{\frac{j}{p}: -R \cdot p \leq j \leq R \cdot p, j \in \integers\}$.
Let $\cd$ be a distribution on $\ci^d$.
Let $f:\reals^d \rightarrow \reals$ be an approximately $\poly(d)$-bounded function, and let $k$ be a constant, namely, independent of $d$.
If $f$ can be approximated by a neural network of depth $k$ and width $\poly(d)$, then it can also be approximated by a neural network of depth $3k+3$, width $\poly(d)$, and $\poly(d)$-bounded weights.
\end{theorem}

\begin{remark}[Constant weights]
Since we require $\poly(d)$ width, then Theorems~\ref{thm:poly weights} and~\ref{thm:poly weights discrete} imply that $f$ can also be approximated by a network of depth $3k+3$ with constant weights (at the expense of a $\poly(d)$ blowup in the width, simply by recursively substituting every neuron with $\poly(d)$-many constant-weight neurons).
\end{remark}

\begin{remark}[Approximation by $\poly(d)$-Lipschitz networks]
Note that from Theorems~\ref{thm:poly weights} and~\ref{thm:poly weights discrete}, it follows that under the assumptions stated there, any  $\poly(d)$-bounded function that can be approximated by a constant-depth network can also be approximated by a constant-depth network that is $\poly(d)$-Lipschitz.
\end{remark}

\begin{remark}[Dependence on $k$]
In Theorems~\ref{thm:poly weights} and~\ref{thm:poly weights discrete}, we obtain a network of depth $3k+3$, width $\poly(d)$, and $\poly(d)$-bounded weights. Since $k$ is a constant, we hide the dependence of the width and of the weights on $k$ inside the $poly()$ notation. We note though that this dependence is exponential in $k$.
\end{remark}

\subsection{Barriers to depth separation}

The proof of Theorem~\ref{thm:poly weights} involves a construction where a network $N$ of depth $k$ is transformed to a network $\hat{N}$ of depth $3k+3$. The network $\hat{N}$ is such that layers $2,\ldots,3k+2$ can be expressed by a threshold circuit. This property enables us to leverage known barriers to depth separation for threshold circuits in order to obtain barriers to depth separation for neural networks.

\begin{theorem}
\label{thm:separation}
Let $\mu$ be a density function on $\reals^d$ with an almost-bounded support and almost-bounded conditional density.
Let $f:\reals^d \rightarrow \reals$ be an approximately $\poly(d)$-bounded function, and let $k' > k \geq 4$ be constants, namely, independent of $d$.
If $f$ cannot be approximated by a neural network of depth $k$ and width $\poly(d)$, but can be approximated by a neural network of depth $k'$ and width $\poly(d)$,
then there is a 
function 
that cannot be computed by a polynomial-sized threshold circuit of depth $k-2$, but can be computed by a 
polynomial-sized threshold circuit of depth $3k'+1$.
\end{theorem}

The main focus in the existing works on depth-separation in neural networks is on continuous input distributions. However, it is also important to study the case where the input distribution is discrete. In the following theorem we show that the barriers to depth separation also hold in this case.

\begin{theorem}
\label{thm:separation discrete}
Let $R(d)$ and $p(d)$ be any polynomials in $d$, and let $\ci = \{\frac{j}{p}: -R p \leq j \leq R p, j \in \integers\}$.
Let $\cd$ be a distribution on $\ci^d$.
Let $f:\reals^d \rightarrow \reals$ be an approximately $\poly(d)$-bounded function, and let $k' > k \geq 4$ be constants, namely, independent of $d$.
If $f$ cannot be approximated by a neural network of width $\poly(d)$ and depth $k$, but can be approximated by a neural network of width $\poly(d)$ and depth $k'$,
then there is a 
function 
that cannot be computed by a 
polynomial-sized threshold circuit of depth $k-2$, but can be computed by a 
polynomial-sized threshold circuit of depth $3k'+1$.
\end{theorem}

\begin{remark}[Barriers to depth separation]
\label{rem:separation}
From Theorems~\ref{thm:separation} and~\ref{thm:separation discrete}, it follows that depth-separation between neural networks of depth $k \geq 4$ and some constant $k'>k$, would imply depth separation between threshold circuits of depth $k-2$ and some constant greater than $k-2$.
Hence, showing depth separation with $k=5$ would solve the longstanding open problem of separating between threshold circuits of depth $3$ and some constant greater than $3$.
Showing depth separation with $k \geq 6$ would solve the open problem of separating between threshold circuits of depth $k-2$ and some constant depth greater than $k-2$, which is especially challenging due to the natural-proof barrier for threshold circuits.
Finally, showing depth separation with $k=4$ would solve the longstanding open problem of separating between threshold circuits of depth $2$ (with arbitrarily large weights) and some constant greater than $2$. Recall that separation between threshold circuits of depth $2$ and $3$ is known only under the assumption that the weight magnitudes are $\poly(d)$ bounded.
\end{remark}

\begin{remark}[Barriers to depth separation with bounded weights]
Sometimes when considering depth separation in neural networks, it is useful to restrict the magnitude of the weights. For example, \cite{daniely2017depth} gave a function that can be approximated by a depth-$3$ network of $\poly(d)$ width and $\poly(d)$-bounded weights, but cannot be approximated by a depth-$2$ network of $\poly(d)$ width and weights bounded by $2^d$. We note that our barrier applies also to this type of separation.
Namely, depth-separation for neural networks of $\poly(d)$-bounded weights between depth $k$ and some constant $k'>k$, would imply depth-separation for threshold circuits of $\poly(d)$-bounded weights between depth $k-2$ and some constant greater than $k-2$. Such separation for threshold circuits is an open problem for circuits of depth at least $3$ (\cite{razborov1992small}), and has a natural-proof barrier for circuits of depth at least $4$ (\cite{krause2001pseudorandom}).
\end{remark}

While Theorems~\ref{thm:separation} and~\ref{thm:separation discrete} give a strong barrier to depth separation, it should not discourage researchers from continuing to investigate the problem, as discussed in the introduction. Moreover, our barrier does not apply to separation between depth $3$ and some larger constant.
However, we now show that even for this case, a depth-separation result would require some different approach than these used in existing results.
As we discussed in Section~\ref{sec:intro}, in the existing depth-separation results for continuous input distributions, $f$ is either a radial function or a function that depends only on one component. In the following theorems we formally show that for such functions, a network of a constant depth greater than $3$ does not have more power than a network of depth $3$
(we note that similar results appeared in e.g., \cite{eldan2016power,daniely2017depth} in the context of specific radial functions, and we actually rely on a technical lemma presented by the former reference).

\begin{theorem}
\label{thm:radial}
Let $\mu$ be a distribution on $\reals^d$ with an almost-bounded support and almost-bounded conditional density.
Let $f:\reals^d \rightarrow \reals$ be an approximately $\poly(d)$-bounded function, that can be approximated by a neural network of $\poly(d)$ width and constant depth.
If $f$ and $\mu$ are radial, then $f$ can be approximated by a network of width $\poly(d)$, depth~$3$, and $\poly(d)$-bounded weights.
\end{theorem}

\begin{theorem}
\label{thm:1d function}
Let $\mu$ be a distribution on $\reals^d$ such that the $d$ components are drawn independently.
Let $f:\reals^d \rightarrow \reals$ be a function that can be approximated by a neural network of $\poly(d)$ width and constant depth.
If $f(\bx)=\sum_{i \in [d]}f_i(x_i)$ for functions $f_i:\reals \rightarrow \reals$, then $f$ can be approximated by a network of width $\poly(d)$ and depth~$2$.
\end{theorem}

\section{Proof ideas}
\label{sec:proof ideas}

In this section we describe the main ideas of the proofs of Theorems~\ref{thm:poly weights},~\ref{thm:poly weights discrete},~\ref{thm:separation} and~\ref{thm:separation discrete}.
For simplicity, in all theorems, instead of assuming that $f$ is approximately $\poly(d)$-bounded, we assume that $f$ is bounded by some $B=\poly(d)$, namely $|f(\bx)| \leq B$ for every $\bx \in \reals^d$.
Also, instead of assuming that $\mu$ has an almost-bounded support, we assume that its support is contained in $[-R,R]^d$ for some $R = \poly(d)$.

\subsection{Neural networks with small weights}
\label{sec:ideas small weights}

We start with the case where the input distribution is discrete (Theorem~\ref{thm:poly weights discrete}) since it is simpler. Then, we describe how to extend it to the continuous case (Theorem~\ref{thm:poly weights}).

\subsubsection{Discrete input distributions}
\label{sec:ideas discrete}

Let $\epsilon = \frac{1}{\poly(d)}$. Let $N$ be a neural network of depth $k$ and width $\poly(d)$ such that $\snorm{N-f}_{L_2(\cd)} \leq \frac{\epsilon}{2}$. Let $N'$ be a network of depth $k+1$ such that for every $\bx \in \reals^d$ we have $N'(\bx) = [N(\bx)]_{[-B,B]}$. Such $N'$ can be obtained from $N$ by adding to it one layer, since
\[
N'(\bx)=[N(\bx)+B]_+ - [N(\bx)-B]_+ - B~.
\]
Note that since $f$ is bounded by $B$, then for every $\bx$ we have $|N'(\bx)-f(\bx)| \leq |N(\bx)-f(\bx)|$, and therefore $\snorm{N'-f}_{L_2(\cd)} \leq \snorm{N-f}_{L_2(\cd)} \leq \frac{\epsilon}{2}$.
We construct a network $\hat{N}$ of constant depth, $\poly(d)$-width and $\poly(d)$-bounded weights, such that $\snorm{\hat{N}-N'}_{L_2(\cd)} \leq \frac{\epsilon}{2}$. Then, we have
\[
\snorm{\hat{N}-f}_{L_2(\cd)} \leq \snorm{\hat{N}-N'}_{L_2(\cd)}+\snorm{N'-f}_{L_2(\cd)} \leq 2 \cdot \frac{\epsilon}{2} = \epsilon~.
\]

Recall that $\cd$ is supported on $\ci^d$, where $\ci = \{\frac{j}{p(d)}: -R p(d) \leq j \leq R p(d), j \in \integers\}$ for some polynomials $p(d),R$.
In order to construct the network $\hat{N}$, we first show the following useful property of $N'$: For every polynomial $p'(d)$, we can construct a network $N''$ of depth $k+1$ and width $\poly(d)$, such that for every $\bx \in \ci^d$ we have $N''(\bx) \in [-B,B]$ and $|N''(\bx)-N'(\bx)| \leq \frac{1}{p'(d)}$, and there exists a positive integer $t \leq 2^{\poly(d)}$ such that all weights and biases in $N''$ are in $Q_t=\{\frac{s}{t}: |s| \leq 2^{\poly(d)}, s \in \integers\}$. Thus, for a sufficiently large polynomial $p'(d)$, we have $\snorm{N''-N'}_{L_2(\cd)} \leq \frac{\epsilon}{2}$, and all weights and biases in $N''$ can be represented by $\poly(d)$ bits.

The idea of the construction of $N''$ is as follows.
First, we transform $N'$ into a network $\cn$ with a special structure (and arbitrary weights) that computes the same function. Then, we define a (very large) system of linear inequalities such that for every $\bx \in \ci^d$ we have inequalities that correspond to the computation $\cn(\bx)$. The variables in the linear system correspond to the weights and biases in $\cn$. Finally, we show that the system has a solution such that all values are in $Q_t$ for some positive integer $t \leq 2^{\poly(d)}$, and that this solution induces a network $N''$ where $|N''(\bx)-N'(\bx)| \leq \frac{1}{p'(d)}$ for every $\bx \in \ci^d$.
We note that a similar idea was used in \cite{maass1997bounds}. However, we use a different construction, since we consider approximation of a real-valued function, while that paper considered exact computation of Boolean functions.

Since the input $\bx \in \ci^d$ is such that for every $i \in [d]$ the component $x_i$ is of the form $\frac{j}{p(d)}$ for some integer $-Rp(d) \leq j \leq Rp(d)$, then $\bx$ can be represented by $\poly(d)$ bits. Hence, the computation of $N''(\bx)$ can be simulated by representing all values, namely, input to neurons, by binary vectors, representing all weights and biases of $N''$ by binary vectors, and computing each layer by applying arithmetic operations, such as multiplication and addition, on the binary vectors. Thus, given an input $\bx$, the network $\hat{N}$ will compute $N''(\bx)$ by simulating $N''$ using arithmetic operations on binary vectors.

It is known that binary multiplication, namely multiplying two $d$-bits binary vectors, and binary iterated addition, namely adding $\poly(d)$ many $d$-bits binary vectors, can be implemented by threshold circuits of $\poly(d)$-width and $\poly(d)$-bounded weights. The depth of the threshold circuit for multiplication is $3$, and the depth of the circuit for iterated addition is $2$ (\cite{siu1994optimal}). Simulating a ReLU of $N''$ can be done by a single layer of threshold gates with small weights.
Hence, simulating $N''(\bx)$ can be done by a threshold circuit $T$ of constant depth, $\poly(d)$ width and $\poly(d)$-bounded weights.
Note that the binary representations of the weights and biases of $N''$ are ``hardwired" into $T$, namely, for each such weight or bias there are gates in $T$ with fan-in $0$ and biases in $\{0,1\}$ that correspond to its binary representation.

Thus, the network $\hat{N}$ consists of three parts:
\begin{enumerate}
\item It transforms the input $\bx$ to a binary representation. Since for every $i \in [d]$ the component $x_i$ is of the form $\frac{j}{p(d)}$ for some integer $-Rp(d) \leq j \leq Rp(d)$, and since $R,p(d)$ are polynomials, then a binary representation of $x_i$ can be computed by two layers of width $\poly(d)$ with $\poly(d)$-bounded weights.
\item It is not hard to show that every threshold circuits of $\poly(d)$ width, $\poly(d)$-bounded weights and depth $m$ can be transformed to a neural network of $\poly(d)$ width, $\poly(d)$-bounded weights and depth $m+1$. By implementing the threshold circuit $T$, the network $\hat{N}$ simulates the computation $N''(\bx)$, and obtains a binary representation of $N''(\bx)$.
\item Finally, since $N''(\bx) \in [-B,B]$ and $B=\poly(d)$, then $N''(\bx)$ can be transformed from a binary representation to its real value while using $\poly(d)$-bounded weights.
\end{enumerate}

\subsubsection{Continuous input distributions}
\label{sec:ideas continuous}

In Section~\ref{sec:ideas discrete}, we described how to approximate a network $N'$ with arbitrary weights by a network $\hat{N}$ with small weights, where the inputs are discrete. In order to handle continuous input distributions, we will first ``round" the input, namely, transform an input $\bx$ to the nearest point $\tbx$ in some discrete set. Then, we will use the construction from Section~\ref{sec:ideas discrete} in order to approximate $N'(\tbx)$. Note that we do not have any guarantees regarding the Lipschitzness of $N'$, and therefore it is possible that $|N'(\bx)-N'(\tbx)|$ is large. Thus, it is not obvious that such a construction approximates $N'$. However, we will show that even though $N'$ is not Lipschitz, $|N'(\bx)-N'(\tbx)|$ is small with high probability over $\bx$. Intuitively, the reason is that $N'(\bx)$ is a bounded function, and has a piecewise-linear structure with a bounded number of pieces along a path. Thus, the measure of the linear segments with a huge Lipschitz constant cannot be too large. Therefore, if we sample $\bx$ and then move from $\bx$ to $\tbx$, the probability that we cross an interval with a huge Lipschitz constant is small.

We now turn to describe the proof ideas in slightly more technical detail. Let $p(d)$ be a polynomial and let $\ci = \{\frac{j}{p(d)}: -R p(d) \leq j \leq R p(d), j \in \integers\}$.
Let $\bx \in [-R,R]^d$. For $i \in [d]$, let $\tx_i \in \ci$ be such that $|\tx_i - x_i|$ is minimal. That is, $\tx_i$ is obtained by rounding $x_i$ to the nearest multiple of $\frac{1}{p(d)}$. Let $\tbx = (\tx_1,\ldots,\tx_d)$. Let $\epsilon = \frac{1}{\poly(d)}$, and let $N$ be a neural network of depth $k$ and width $\poly(d)$ such that $\snorm{N-f}_{L_2(\mu)} \leq \frac{\epsilon}{3}$. Let $N'$ be a network of depth $k+1$ and width $\poly(d)$ such that for every $\bx \in \reals^d$ we have $N'(\bx) = [N(\bx)]_{[-B,B]}$. Since $f$ is bounded by $B$, then for every $\bx$ we have $|N'(\bx)-f(\bx)| \leq |N(\bx)-f(\bx)|$, and therefore $\snorm{N'-f}_{L_2(\mu)} \leq \snorm{N-f}_{L_2(\mu)} \leq \frac{\epsilon}{3}$. Let $\tilde{N}:\reals^d \rightarrow [-B,B]$ be a function such that for every $\bx \in [-R,R]^d$ we have $\tilde{N}(\bx)=N'(\tbx)$.
We will show that $\snorm{\tilde{N}-N'}_{L_2(\mu)} \leq \frac{\epsilon}{3}$, and then construct a network $\hat{N}$ of constant depth, $\poly(d)$-width and $\poly(d)$-bounded weights, such that $\snorm{\hat{N}-\tilde{N}}_{L_2(\mu)} \leq \frac{\epsilon}{3}$. Thus, we have
\[
\snorm{\hat{N}-f}_{L_2(\mu)} \leq \snorm{\hat{N}-\tilde{N}}_{L_2(\mu)} + \snorm{\tilde{N}-N'}_{L_2(\mu)} + \snorm{N'-f}_{L_2(\mu)} \leq 3 \cdot \frac{\epsilon}{3} = \epsilon~.
\]

We start with $\snorm{\tilde{N}-N'}_{L_2(\mu)} \leq \frac{\epsilon}{3}$. Since $N'$ is bounded by $B$, then for every $\bx \in [-R,R]^d$ we have $|\tilde{N}(\bx)-N'(\bx)| \leq 2B$. In order to bound $\snorm{\tilde{N}-N'}_{L_2(\mu)}$ we need to show that w.h.p. $|\tilde{N}(\bx)-N'(\bx)|$ is small. Namely, that w.h.p. the value of $N'$ does not change too much by moving from $\bx$ to $\tbx$. Since the Lipschitzness of $N'$ is not bounded, then for every choice of a polynomial $p(d)$, we are not guaranteed that $|N'(\tbx)-N'(\bx)|$ is small. Hence, it is surprising that we can bound $\snorm{\tilde{N}-N'}_{L_2(\mu)}$ by $\frac{\epsilon}{3}$.
We show that while it is possible that $|N'(\tbx)-N'(\bx)|$ is not small, if $p(d)$ is a sufficiently large polynomial then the probability of such an event, when $\bx$ is drawn according to $\mu$, is small.
Intuitively, it follows from the following argument.
We move from $\bx$ to $\tbx$ in $d$ steps. In the $i$-th step we change the $i$-th component from $x_i$ to $\tx_i$. Namely, we move from $(\tx_1,\ldots,\tx_{i-1},x_i,x_{i+1},\ldots,x_d)$ to $(\tx_1,\ldots,\tx_{i-1},\tx_i,x_{i+1},\ldots,x_d)$. We will show that for each step, w.h.p., the change in $N'$ is small. Since in the $i$-th step the components $[d] \setminus \{i\}$ are fixed, then the dependence of $N'$ on the value of the $i$-th component, which is the component that we change, can be expressed by a network with input dimension $1$, width $\poly(d)$, and constant depth. Such a network computes a function $g_i:\reals \rightarrow \reals$ that is piecewise linear with $\poly(d)$ pieces. Since $N'$ is bounded by $B$ then $g_i$ is also bounded by $B$, and therefore a linear piece in $g_i$ whose derivative has a large absolute value is supported on a small interval. Now, we are able to show that w.h.p. the interval between $x_i$ and $\tx_i$ has an empty intersection with intervals of $g_i$ whose derivatives have large absolute values. Hence, w.h.p. the change in the value of $N'$ in the $i$-th step is small.

We now describe the network $\hat{N}$ such that $\snorm{\hat{N}-\tilde{N}}_{L_2(\mu)} \leq \frac{\epsilon}{3}$.
First, the network $\hat{N}$ transforms w.h.p. the input $\bx$ to $\tbx$. Note that the mapping $\bx \mapsto \tbx$ is not continuous and hence cannot be computed by a neural network for all $\bx \in [-R,R]^d$, but it is not hard to construct a network $\cn_1$ of depth $2$, width $\poly(d)$ and $\poly(d)$-bounded weights that computes it w.h.p., where $\bx$ is drawn according to $\mu$.
Now, by the same arguments described in Section~\ref{sec:ideas discrete} for the case of a discrete input distributions, for every polynomial $p'(d)$ there is a network $\cn_2$ of constant depth, $\poly(d)$ width and $\poly(d)$-bounded weights, such that for every $\tbx \in \ci^d$ we have $|\cn_2(\tbx)-N'(\tbx)| \leq \frac{1}{p'(d)}$.
Let $\hat{N}$ be the composition of $\cn_1$ and $\cn_2$. Now, we have w.h.p. that $|\hat{N}(\bx) - \tilde{N}(\bx)| = |\hat{N}(\bx)-N'(\tbx)| \leq \frac{1}{p'(d)}$. Also, both $\hat{N}$ and $\tilde{N}$ are bounded by $B$ and therefore for every $\bx \in [-R,R]^d$ we have $|\hat{N}(\bx)-\tilde{N}(\bx)| \leq 2B$.
Hence, for a sufficiently large polynomial $p'(d)$ we have $\snorm{\hat{N}-\tilde{N}}_{L_2(\mu)} \leq \frac{\epsilon}{3}$.

\subsection{Barriers to depth separation}

We now describe the idea behind the proofs of Theorems~\ref{thm:separation} and~\ref{thm:separation discrete}.
We consider here only continuous input distributions, but the case of discrete input distributions is similar.

Let $\epsilon = \frac{1}{\poly(d)}$, and let $N$ be a neural network of depth $k'$ such that $\snorm{N-f}_{L_2(\mu)} \leq \frac{\epsilon}{3}$. In Section~\ref{sec:ideas small weights} we described a construction of a network $\hat{N}$ of a constant depth and $\poly(d)$ width, such that $\snorm{\hat{N}-f}_{L_2(\mu)} \leq \epsilon$.
The network $\hat{N}$ is such that in the first layers it transforms the input $\bx$ to a binary representation of $\tbx$, then it computes $T(\tbx)$ for an appropriate threshold circuit $T$, and finally it transforms the output from a binary representation to its real value.
In the proof of Theorem~\ref{thm:poly weights} we describe this construction in more detail and show that $\hat{N}$ is of depth $3k'+3$, and $T$ is of depth $3k'+1$.

Let $g$
be the function that $T$ computes. 
Assume that $g$ can be computed by a threshold circuit $T'$ of depth $k-2$ and width $\poly(d)$. Now, as we show, we can replace the layers in $\hat{N}$ that simulate $T$ by layers that simulate $T'$, and obtain a network $\hat{\cn}$ of depth $k$. Also, since $T$ and $T'$ compute the same function, then $\snorm{\hat{\cn}-f}_{L_2(\mu)} \leq \epsilon$. Hence, $f$ can be approximated by a network of depth $k$, in contradiction to the assumption. This implies that $g$ \emph{cannot} be computed by a threshold circuits of depth $k-2$, hence establishing a depth separation property for threshold circuits.

\section{Proofs}\label{sec:proofs}

\subsection{Proof of Theorem~\ref{thm:poly weights}}

Let $\epsilon=\frac{1}{\poly(d)}$.
Let $N$ be a neural network of depth $k$ and width $\poly(d)$, such that $\snorm{N-f}_{L_2(\mu)} \leq \frac{\epsilon}{5}$.
We will construct a network $\hat{N}$ of depth $3k+3$, width $\poly(d)$ and $\poly(d)$-bounded weights, such that $\snorm{\hat{N}-f}_{L_2(\mu)} \leq \epsilon$.

Since $f$ is approximately $\poly(d)$-bounded, there is $B=\poly(d)$ be such that
\[
\E_{\bx \sim \mu}\left(f(\bx)-[f(\bx)]_{[-B,B]}\right)^2 \leq \left(\frac{\epsilon}{5}\right)^2~.
\]
Let $f':\reals^d \rightarrow \reals$ be such that $f'(\bx)=[f(\bx)]_{[-B,B]}$. Thus,
\begin{equation}
\label{eq:triangle1}
\snorm{f'-f}_{L_2(\mu)} \leq \frac{\epsilon}{5}~.
\end{equation}
Let $N'$ be a network of depth $k+1$ such that for every $\bx \in \reals^d$ we have $N'(\bx)=[N(\bx)]_{[-B,B]}$. Such $N'$ can be obtained from $N$ by adding to it one layer, since $N'(\bx)=[N(\bx)+B]_+ - [N(\bx)-B]_+ - B$. Although we do not allow bias in the output neuron, the additive term $-B$ can be implemented by adding a hidden neuron with fan-in $0$ and bias $1$, that is connected to the output neuron with weight $-B$.
Note that $\E_{\bx \sim \mu}(N'(\bx)-f'(\bx))^2 \leq \E_{\bx \sim \mu}(N(\bx)-f'(\bx))^2$, and therefore
\begin{equation}
\label{eq:triangle2}
\snorm{N'-f'}_{L_2(\mu)} \leq \snorm{N-f'}_{L_2(\mu)} \leq \snorm{N-f}_{L_2(\mu)} + \snorm{f-f'}_{L_2(\mu)} \leq \frac{2 \epsilon}{5}~.
\end{equation}

Let $\delta = \frac{\epsilon^2}{400B^2}$. Since $\mu$ has an almost-bounded support, there is $R = \poly(d)$ such that $Pr_{\bx \sim \mu}(\bx \not \in [-R,R]^d) \leq \delta$.
Let $p(d)$ be a polynomial. Let $\ci = \{\frac{j}{p(d)}: -R p(d) \leq j \leq R p(d), j \in \integers\}$.
Let $\bx \in \reals^d$. For every $i$ such that $x_i \in [-R-\frac{1}{2p(d)},R+\frac{1}{2p(d)}]$, let $\tx_i \in \ci$ be such that $|\tx_i - x_i|$ is minimal. That is, $\tx_i$ is obtained by rounding $x_i$ to the nearest multiple of $\frac{1}{p(d)}$. For every $i$ such that $x_i \not \in [-R-\frac{1}{2p(d)},R+\frac{1}{2p(d)}]$, let $\tx_i=0$. Then, let $\tbx = (\tx_1,\ldots,\tx_d)$.
Let $\tilde{N}:\reals^d \rightarrow [-B,B]$ be a function such that for every $\bx \in \reals^d$ we have $\tilde{N}(\bx)=N'(\tbx)$.
We will prove the following two lemmas:

\begin{lemma}
\label{lemma:triangle3}
There exists a polynomial $p(d)$ such that
\[
\snorm{\tilde{N}-N'}_{L_2(\mu)} \leq \frac{\epsilon}{5}~.
\]
\end{lemma}

\begin{lemma}
\label{lemma:triangle4}
There exists a neural network $\hat{N}$ of depth $3k+3$, width $\poly(d)$ and $\poly(d)$-bounded weights, such that
\[
\snorm{\hat{N}-\tilde{N}}_{L_2(\mu)} \leq \frac{\epsilon}{5}~.
\]
\end{lemma}

Then, combining Lemmas~\ref{lemma:triangle3} and~\ref{lemma:triangle4} with Eq.~\ref{eq:triangle1} and~\ref{eq:triangle2}, we have
\begin{align*}
\snorm{\hat{N}-f}_{L_2(\mu)}
&\leq \snorm{\hat{N}-\tilde{N}}_{L_2(\mu)} + \snorm{\tilde{N}-N'}_{L_2(\mu)} + \snorm{N'-f'}_{L_2(\mu)} + \snorm{f'-f}_{L_2(\mu)}
\\
&\leq \frac{\epsilon}{5} + \frac{\epsilon}{5} + \frac{2\epsilon}{5} + \frac{\epsilon}{5}
= \epsilon~,
\end{align*}
and hence the theorem follows.

\subsubsection{Proof of Lemma~\ref{lemma:triangle3}}

We start with an intuitive explanation, and then turn to the formal proof.
Since we have $\tilde{N}(\bx)=N'(\tbx)$ and $|N'(\bx)| \leq B$ for every $\bx$, then we have $|\tilde{N}(\bx)-N'(\bx)| \leq 2B$. In order to bound $\snorm{\tilde{N}-N'}_{L_2(\mu)}$ we show that w.h.p. $|\tilde{N}(\bx)-N'(\bx)|$ is small. Namely, that w.h.p. the value of $N'$ does not change too much by moving from $\bx$ to $\tbx$. Since the Lipschitzness of $N'$ is not bounded, then for every choice of a polynomial $p(d)$, we are not guaranteed that $|N'(\tbx)-N'(\bx)|$ is small. However, we show that for a sufficiently large polynomial $p(d)$, the probability that we encounter a region where $N'$ has large derivative while moving from $\bx$ to $\tbx$, is small.

We move from $\bx$ to $\tbx$ in $d$ steps. In the $i$-th step we change the $i$-th component from $x_i$ to $\tx_i$. Namely, we move from $(\tx_1,\ldots,\tx_{i-1},x_i,x_{i+1},\ldots,x_d)$ to $(\tx_1,\ldots,\tx_{i-1},\tx_i,x_{i+1},\ldots,x_d)$. We show that for each step, w.h.p., the change in $N'$ is small. Since in the $i$-th step the components $[d] \setminus \{i\}$ are fixed, then the dependence of $N'$ on the value of the $i$-th component, which is the component that we change, can be expressed by a network with input dimension $1$, width $\poly(d)$, and constant depth. Such a network computes a function $g_i:\reals \rightarrow \reals$ that is piecewise linear with $\poly(d)$ pieces. Since $N'$ is bounded by $B$ then $g_i$ is also bounded by $B$, and therefore a linear piece in $g_i$ whose derivative has a large absolute value is supported on a small interval. Now, we need to show that the interval between $x_i$ and $\tx_i$ has an empty intersection with intervals of $g_i$ with large derivatives. Since there are only $\poly(d)$ intervals and intervals with large derivatives are small, then by using the fact that $\mu$ has an almost-bounded conditional density, we are able to show that w.h.p. the interval between $x_i$ and $\tx_i$ does not have a non-empty intersection with such intervals. Intuitively, we can think about the choice of $\bx \sim \mu$ as choosing the components $[d] \setminus \{i\}$ according to $\mu_{[d] \setminus i}$ and then choosing $x_i$ according to $\mu_{i | [d] \setminus i}$. Now, the choice of the components $[d] \setminus \{i\}$ induces the function $g_i$, and the choice of $x_i$ is good with respect to $g_i$ if the interval between $x_i$ and $\tx_i$ does not have a non-empty intersection with the intervals of $g_i$ which have large derivatives. We show that w.h.p. we obtain $g_i$ and $x_i$, such that $x_i$ is good with respect to $g_i$.

We now turn to the formal proof.
Let
\[
A = \left\{\bx \in \reals^d: (N'(\bx)-\tilde{N}(\bx))^2 > \frac{\epsilon^2}{50}\right\}~.
\]
Let $\bx \in \reals^d$, let $\by_i = (\tx_1,\ldots,\tx_{i-1},x_{i+1},\ldots,x_d) \in \reals^{d-1}$, and let $N'_{i,\by_i}:\reals \rightarrow [-B,B]$ be such that
\[
N'_{i,\by_i}(t) = N'(\tx_1,\ldots,\tx_{i-1},t,\ldots,x_d)~.
\]

Note that for every $\bx \in \reals^d$, we have
\begin{align*}
|N'(\bx)-\tilde{N}(\bx)|
&= |N'(\bx)-N'(\tbx)|
\\
&= \left|\sum_{i \in [d]}N'(\tx_1,\ldots,\tx_{i-1},x_i,\ldots,x_d)-N'(\tx_1,\ldots,\tx_i,x_{i+1},\ldots,x_d)\right|
\\
&\leq \sum_{i \in [d]}|N'(\tx_1,\ldots,\tx_{i-1},x_i,\ldots,x_d)-N'(\tx_1,\ldots,\tx_i,x_{i+1},\ldots,x_d)|
\\
&= \sum_{i \in [d]}|N'_{i,\by_i}(x_i)-N'_{i,\by_i}(\tx_i)|~.
\end{align*}

Thus, using the shorthand $Pr(\cdot \; | \; [-R,R]^d)$ for $Pr(\cdot \; | \; \bx \in [-R,R]^d)$, we have
\begin{align*}
Pr(A \; | \; \bx \in [-R,R]^d)
&= Pr\left( (N'(\bx)-\tilde{N}(\bx))^2 > \frac{\epsilon^2}{50} \middle| [-R,R]^d \right)
\\
&= Pr\left( |N'(\bx)-\tilde{N}(\bx)| > \frac{\epsilon}{5\sqrt{2}} \middle| [-R,R]^d \right)
\\
&\leq Pr\left(\sum_{i \in [d]}|N'_{i,\by_i}(x_i)-N'_{i,\by_i}(\tx_i)| > \frac{\epsilon}{5\sqrt{2}} \middle| [-R,R]^d \right)
\\
&\leq Pr\left(\exists i \in [d] \text{\ s.t.\ } |N'_{i,\by_i}(x_i)-N'_{i,\by_i}(\tx_i)| > \frac{\epsilon}{5\sqrt{2}d} \middle| [-R,R]^d \right)
\\
&\leq \sum_{i \in [d]} Pr\left(|N'_{i,\by_i}(x_i)-N'_{i,\by_i}(\tx_i)| > \frac{\epsilon}{5\sqrt{2}d} \middle| [-R,R]^d \right)~.
\end{align*}

Now, since $Pr(\bx \not \in [-R,R]^d) \leq \delta$, we have
\begin{align}
\label{eq:A}
Pr(A)
&= Pr(A \; | \; [-R,R]^d) \cdot Pr([-R,R]^d) + Pr(A \; | \; \reals^d \setminus [-R,R]^d) \cdot Pr(\reals^d \setminus [-R,R]^d) \nonumber
\\
&\leq \sum_{i \in [d]} Pr\left(|N'_{i,\by_i}(x_i)-N'_{i,\by_i}(\tx_i)| > \frac{\epsilon}{5\sqrt{2}d} \middle| [-R,R]^d \right) + \delta~.
\end{align}

Let
\[
A_i = \left\{\bx \in \reals^d: |N'_{i,\by_i}(x_i)-N'_{i,\by_i}(\tx_i)| > \frac{\epsilon}{5\sqrt{2}d}\right\}~.
\]

\begin{lemma}
\label{lemma:A_i}
\[
Pr\left( A_i \; | \; \bx \in [-R,R]^d \right)
\leq \frac{\epsilon^2}{400B^2d}~.
\]
\end{lemma}

\begin{proof}
Let $\delta' = \frac{\epsilon^2}{1600B^2d}$. Since $\mu$ has an almost-bounded conditional density, there is $M = \poly(d)$ such that for every $i \in [d]$ we have $Pr(G_i) \leq \delta'$, where
\[
G_i = \{\bx \in \reals^d: \exists t \in \reals \text{\ s.t.\ } \mu_{i | [d] \setminus i}(t | x_1,\ldots,x_{i-1},x_{i+1},\ldots,x_d) > M \}~.
\]
Now, we have
\begin{align}
\label{eq:Ai}
Pr\left(A_i | [-R,R]^d \right)
&= Pr\left( A_i \cap G_i | [-R,R]^d \right) + Pr\left( A_i \cap (\reals^d \setminus G_i) | [-R,R]^d \right) \nonumber
\\
&\leq Pr\left( G_i | [-R,R]^d \right) + Pr\left( A_i \cap (\reals^d \setminus G_i) | [-R,R]^d \right)
\end{align}

Note that
\begin{equation}
\label{eq:Gi}
Pr\left(G_i | [-R,R]^d \right)
= \frac{Pr\left(G_i \cap [-R,R]^d \right)}{Pr\left([-R,R]^d\right)}
\leq \frac{Pr\left(G_i\right)}{Pr\left([-R,R]^d\right)}
\leq \frac{\delta'}{1-\delta}~.
\end{equation}

Thus, it remains to bound $Pr\left( A_i \cap (\reals^d \setminus G_i) | [-R,R]^d \right)$.
Let $\bx \in [-R,R]^d$. Note that the function $N'_{i,\by_i}:\reals \rightarrow \reals$ can be expressed by a neural network of depth $k+1$ that is obtained from $N'$ by using the hardwired $\by_i$ instead of the corresponding input components $[d] \setminus \{i\}$. That is, if a neuron in the first hidden layer of $N'$ has weights $w_1,\ldots,w_d$ and bias $b$, then in $N'_{i,\by_i}$ its weight is $w_i$ and its bias is $b + \inner{(w_1,\ldots,w_{i-1},w_{i+1},\ldots,w_d),\by_i}$.
A neural network with input dimension $1$, constant depth, and $m$ neurons in each hidden layer, is piecewise linear with at most $\poly(m)$ pieces (\cite{telgarsky2015representation}). Therefore, $N'_{i,\by_i}$ consists of $l=\poly(d)$ linear pieces. Note that $l$ depends only on the depth and width of $N'$, and does not depend on $i$ and $\by_i$.

Since $N'_{i,\by_i}(t) \in [-B,B]$ for every $t \in \reals$, then if $N'_{i,\by_i}$ has derivative $\alpha$ in a linear interval $[a,b]$ then $|(b-a)\alpha| \leq 2B$.
Let $\bx^i = (x_1,\ldots,x_{i-1},x_{i+1},\ldots,x_d)\in \reals^{d-1}$.
Let $\gamma = \frac{6400B^3dMl}{\epsilon^2}$. We denote by $I_{i,\bx^i,\gamma}$ the set of intervals $[a_j,b_j]$ where the derivative $\alpha_j$ in $N'_{i,\by_i}$ satisfies $|\alpha_j| > \gamma$. Note that $\by_i$ depends on $\bx^i$ and does not depend on $x_i$. Now,
\begin{equation}
\label{eq:ellI}
\sum_{[a_j,b_j] \in I_{i,\bx^i,\gamma}}(b_j-a_j) \leq \sum_{[a_j,b_j] \in I_{i,\bx^i,\gamma}}\frac{2B}{|\alpha_j|} < l \cdot \frac{2B}{\gamma}~.
\end{equation}
Let $\beta$ be the open interval $(x_i,\tx_i)$ if $x_i \leq \tx_i$ or $(\tx_i,x_i)$ otherwise. If $\beta \cap [a_j,b_j] = \emptyset$ for every $[a_j,b_j] \in I_{i,\bx^i,\gamma}$, then
\[
|N'_{i,\by_i}(x_i) - N'_{i,\by_i}(\tx_i)| \leq |\tx_i - x_i|\gamma \leq \frac{\gamma}{p(d)}~.
\]
Let
\[
I'_{i,\bx^i,\gamma} = \left\{[a'_j,b'_j]: a'_j=a_j-\frac{1}{p(d)}, b'_j=b_j+\frac{1}{p(d)}, [a_j,b_j] \in I_{i,\bx^i,\gamma}\right\}~.
\]
Thus, if $|N'_{i,\by_i}(x_i) - N'_{i,\by_i}(\tx_i)| > \frac{\gamma}{p(d)}$ then $\beta \cap [a_j,b_j] \neq \emptyset$ for some $[a_j,b_j] \in I_{i,\bx^i,\gamma}$, and therefore $x_i \in [a'_j,b'_j]$ for some $[a'_j,b'_j] \in I'_{i,\bx^i,\gamma}$.
Hence, for a sufficiently large polynomial $p(d)$, if $\bx \in A_i$ then $x_i \in [a'_j,b'_j]$ for some $[a'_j,b'_j] \in I'_{i,\bx^i,\gamma}$.

We denote $\ell(I'_{i,\bx^i,\gamma}) = \sum_{[a'_j,b'_j] \in I'_{i,\bx^i,\gamma}}(b'_j-a'_j)$.
Note that
\begin{align}
\label{eq:ellIprime}
\ell(I'_{i,\bx^i,\gamma})
= \sum_{[a_j,b_j] \in I_{i,\bx^i,\gamma}}(b_j-a_j+\frac{2}{p(d)})
\leq \frac{2l}{p(d)} + \sum_{[a_j,b_j] \in I_{i,\bx^i,\gamma}}(b_j-a_j)
\stackrel{(Eq.~\ref{eq:ellI})}{<} \frac{2l}{p(d)} + \frac{2Bl}{\gamma}~.
\end{align}

For $\bz \in \reals^{d-1}$ and $t \in \reals$ we denote $\bz_{i,t} = (z_1,\ldots,z_{i-1},t,z_i,\ldots,z_{d-1}) \in \reals^d$. Note that for every $\bz \in \reals^{d-1}$ and every $t,t' \in \reals$, we have $\bz_{i,t} \in G_i$ iff $\bz_{i,t'} \in G_i$.
Let $G'_i = \{\bz \in \reals^{d-1}: \exists t \in \reals \text{\ s.t.\ } \bz_{i,t} \in G_i\}$.
Now, we have
\begin{align*}
Pr(A_i\cap (\reals^d \setminus G_i) \cap [-R,R]^d)
&=\int_{A_i \cap (\reals^d \setminus G_i) \cap [-R,R]^d} \mu(\bx) d\bx
\\
&= \int_{\reals^{d-1} \setminus G'_i} \left[ \int_{\{t \in \reals: \bz_{i,t} \in A_i \cap [-R,R]^d\}} \mu(\bz_{i,t}) dt \right] d\bz
\\
&= \int_{\reals^{d-1} \setminus G'_i} \left[ \int_{\{t \in \reals: \bz_{i,t} \in A_i \cap [-R,R]^d\}} \mu_{[d] \setminus i}(\bz) \mu_{i | [d] \setminus i}(t | \bz) dt \right] d\bz
\\
&= \int_{\reals^{d-1} \setminus G'_i} \mu_{[d] \setminus i}(\bz) \left[ \int_{\{t \in \reals: \bz_{i,t} \in A_i \cap [-R,R]^d\}} \mu_{i | [d] \setminus i}(t | \bz) dt \right] d\bz
\\
&\leq \sup_{\bz \in \reals^{d-1} \setminus G'_i} \int_{\{t \in \reals: \bz_{i,t} \in A_i \cap [-R,R]^d\}} \mu_{i | [d] \setminus i}(t | \bz) dt~.
\end{align*}
Recall that if $\bz \in \reals^{d-1} \setminus G'_i$ then $\mu_{i | [d] \setminus i}(t | \bz) \leq M$ for all $t \in \reals$. Hence the above is at most
\[
\sup_{\bz \in \reals^{d-1} \setminus G'_i} \int_{\{t \in \reals: \bz_{i,t} \in A_i \cap [-R,R]^d\}} M dt~.
\]
Also, recall that if $\bx \in A_i \cap [-R,R]^d$ then $x_i \in [a'_j,b'_j]$ for some $[a'_j,b'_j] \in I'_{i,\bx^i,\gamma}$. Therefore the above is at most
\begin{align*}
\sup_{\bz \in \reals^{d-1} \setminus G'_i} \int_{\{t \in [a'_j,b'_j]: [a'_j,b'_j] \in I'_{i,\bz,\gamma}\}} M dt
&\leq \sup_{\bz \in \reals^{d-1} \setminus G'_i} \sum_{[a'_j,b'_j] \in I'_{i,\bz,\gamma}}\int_{[a'_j,b'_j]} M dt
\\
&= \sup_{\bz \in \reals^{d-1} \setminus G'_i} \sum_{[a'_j,b'_j] \in I'_{i,\bz,\gamma}}(b'_j-a'_j)M
\\
&= \sup_{\bz \in \reals^{d-1} \setminus G'_i} M \cdot \ell(I'_{i,\bz,\gamma})
\stackrel{(Eq.~\ref{eq:ellIprime})}{<} M\left(\frac{2l}{p(d)} + \frac{2Bl}{\gamma}\right)~.
\end{align*}

Now, we have
\begin{align*}
Pr\left( A_i \cap (\reals^d \setminus G_i) | [-R,R]^d \right)
&= \frac{Pr\left(A_i \cap (\reals^d \setminus G_i) \cap [-R,R]^d\right)}{Pr([-R,R]^d)}
\\
&\leq M\left(\frac{2l}{p(d)} + \frac{2Bl}{\gamma}\right)\frac{1}{1-\delta}~.
\end{align*}

Combining the above with Eq.~\ref{eq:Ai} and~\ref{eq:Gi}, and using $\delta \leq \frac{1}{2}$, $\delta' = \frac{\epsilon^2}{1600B^2d}$ and $\gamma = \frac{6400B^3dMl}{\epsilon^2}$, we have
\begin{align*}
Pr\left(A_i | [-R,R]^d \right)
&\leq \frac{\delta'}{1-\delta} + M\left(\frac{2l}{p(d)} + \frac{2Bl}{\gamma}\right)\frac{1}{1-\delta}
\leq 2\delta' + 2M\left(\frac{2l}{p(d)} + \frac{2Bl}{\gamma}\right)
\\
&= \frac{\epsilon^2}{800B^2d} + \frac{4Ml}{p(d)} + \frac{\epsilon^2}{1600B^2d}~.
\end{align*}
Therefore, for a sufficiently large polynomial $p(d)$ we have $Pr\left(A_i | [-R,R]^d \right) \leq \frac{\epsilon^2}{400B^2d}$.
\end{proof}

By combining Lemma~\ref{lemma:A_i} and Eq.~\ref{eq:A}, and plugging in $\delta = \frac{\epsilon^2}{400B^2}$ we have
\begin{align*}
Pr(A)
&\leq \sum_{i \in [d]} \frac{\epsilon^2}{400B^2d} + \frac{\epsilon^2}{400B^2}
= \frac{\epsilon^2}{200B^2}~.
\end{align*}

Finally, since for every $\bx$ we have $(N'(\bx)-\tilde{N}(\bx))^2 \leq (2B)^2$, and since for every $\bx \not \in A$ we have $(N'(\bx)-\tilde{N}(\bx))^2 \leq \frac{\epsilon^2}{50}$, then
\begin{align*}
\E_{\bx \sim \mu}\left(N'(\bx)-\tilde{N}(\bx)\right)^2
&\leq Pr(A) \cdot (2B)^2 + Pr(\reals^d \setminus A) \cdot \frac{\epsilon^2}{50}
\leq \frac{\epsilon^2}{200B^2} \cdot 4B^2 + \frac{\epsilon^2}{50}
= \left(\frac{\epsilon}{5}\right)^2~.
\end{align*}

\subsubsection{Proof of Lemma~\ref{lemma:triangle4}}

The network $\hat{N}$ consists of three parts. First, it transforms with high probability the input $\bx$ to a binary representation of $\tbx$. Then, it simulates $N(\tbx)$ by using arithmetic operations on binary vectors. Finally, it performs clipping of the output to the interval $[-B,B]$ and transforms it from the binary representation to its real value.

We start with the first part of $\hat{N}$, namely, transforming the input $\bx$ to a binary representation of $\tbx$. The following lemma shows a property of almost-bounded conditional densities, that is required for this transformation.

\begin{lemma}
\label{lemma:from bounded conditional to marginal}
Let $\mu$ be a distribution with an almost-bounded conditional density. Then, for every $\epsilon=\frac{1}{\poly(d)}$ there is $\Delta=\frac{1}{\poly(d)}$ such that for every $i \in [d]$ and $s \in \reals$ we have
\[
Pr_{\bx \sim \mu}\left(x_i \in [s,s+\Delta]\right) \leq \epsilon~.
\]
\end{lemma}
\begin{proof}
For $\bx \in \reals^d$ we denote $\bx^i = (x_1,\ldots,x_{i-1},x_{i+1},\ldots,x_d) \in \reals^{d-1}$.
Since $\mu$ has an almost-bounded conditional density, then there is $M=\poly(d)$ such that for every $i \in [d]$ we have
\[
Pr_{\bx \sim \mu}\left(\exists t \in \reals \text{\ s.t.\ } \mu_{i | [d] \setminus i}(t | \bx^i) > M\right)   \leq \frac{\epsilon}{2}~.
\]
Let $\Delta = \frac{1}{\poly(d)}$ such that $M \Delta \leq \frac{\epsilon}{2}$.

Then,
\begin{align*}
Pr_{\bx \sim \mu}&(x_i \in [s,s + \Delta])
= \int_{\{\bx: x_i \in [s,s + \Delta]\}} \mu(\bx) d\bx
\\
&= \int_{\{\bx: x_i \in [s,s + \Delta]\}} \mu_{[d]\setminus i}(\bx^i)\mu_{i|[d]\setminus i}(x_i|\bx^i) d\bx
\\
&= \int_{\reals^{d-1}}\mu_{[d]\setminus i}(\bx^i) \left[\int_{[s,s + \Delta]} \mu_{i|[d]\setminus i}(x_i|\bx^i)dx_i\right] d\bx^i
\\
&= \int_{\{\bx^i \in \reals^{d-1}: \forall t \in \reals \;.\; \mu_{i | [d] \setminus i}(t | \bx^i) \leq M \}}\mu_{[d]\setminus i}(\bx^i) \left[\int_{[s,s + \Delta]} \mu_{i|[d]\setminus i}(x_i|\bx^i)dx_i\right] d\bx^i + \\
&\phantom{{}= x} \int_{\{\bx^i \in \reals^{d-1}: \exists t \in \reals \;.\; \mu_{i | [d] \setminus i}(t | \bx^i) > M \}}\mu_{[d]\setminus i}(\bx^i) \left[\int_{[s,s + \Delta]} \mu_{i|[d]\setminus i}(x_i|\bx^i)dx_i\right] d\bx^i
\\
&\leq \int_{\{\bx^i \in \reals^{d-1}: \forall t \in \reals \;.\; \mu_{i | [d] \setminus i}(t | \bx^i) \leq M \}}\mu_{[d]\setminus i}(\bx^i) \left[\int_{[s,s + \Delta]} M dx_i\right] d\bx^i + \\
&\phantom{{}= x} \int_{\{\bx^i \in \reals^{d-1}: \exists t \in \reals \;.\; \mu_{i | [d] \setminus i}(t | \bx^i) > M \}}\mu_{[d]\setminus i}(\bx^i) \left[\int_{\reals} \mu_{i|[d]\setminus i}(x_i|\bx^i)dx_i\right] d\bx^i
\\
&\leq \int_{\{\bx^i \in \reals^{d-1}: \forall t \in \reals \;.\; \mu_{i | [d] \setminus i}(t | \bx^i) \leq M \}}\mu_{[d]\setminus i}(\bx^i) M\Delta d\bx^i + \\
&\phantom{{}= x} \int_{\{ \bx \in \reals^d: \exists t \in \reals \;.\; \mu_{i | [d] \setminus i}(t | \bx^i) > M \}} \mu(\bx) d\bx
\\
&\leq M \Delta + Pr_{\bx \sim \mu}\left( \exists t \in \reals \text{\ s.t.\ } \mu_{i | [d] \setminus i}(t | \bx^i) > M \right)
\leq \frac{\epsilon}{2} + \frac{\epsilon}{2}
= \epsilon~.
\end{align*}
\end{proof}

Let $c$ be an integer greater or equal to $\log(2Rp(d)+1)$. For $i \in [d]$ we denote by $(p(d)\tx_i)^{\bin(c)} \in \{0,1\}^c$ the $c$-bits binary representation of the integer $p(d)\tx_i$. Note that since $p(d)\tx_i \in [-Rp(d),Rp(d)]$ then $c$ bits are sufficient. We use the standard {\em two's complement} binary representation.
In this representation, the arithmetic operations of addition and multiplication of signed numbers are identical to those for unsigned numbers. Thus, we do not need to handle negative and positive numbers differently.
We denote by $(p(d)\tbx)^{\bin(c)} \in \{0,1\}^{c \cdot d}$ the binary representation of $p(d)\tbx$, obtained by concatenating $(p(d)\tx_i)^{\bin(c)}$ for $i=1,\ldots,d$.

\begin{lemma}
\label{lemma:from bx to binary}
Let $c \leq \poly(d)$ be an integer greater or equal to $\log(2Rp(d)+1)$ and let $\delta'=\frac{1}{\poly(d)}$. There is a neural network $\cn$ of depth $2$, width $\poly(d)$, weights bounded by some $\poly(d)$, and $(c \cdot d)$ outputs, such that
\[
Pr_{\bx \sim \mu}\left(\cn(\bx) = (p(d)\tbx)^{\bin(c)}\right) \geq 1 - \delta'~.
\]
\end{lemma}
\begin{proof}
In order to construct $\cn$, we need to show how to compute $(p(d)\tx_i)^{\bin(c)}$ for every $i \in [d]$.
We will show a depth-$2$ network $\cn'$ such that given $x_i \sim \mu_i$ it outputs $(p(d)\tx_i)^{\bin(c)}$ w.p. $\geq 1 - \frac{\delta'}{d}$. Then, the network $\cn$ consists of $d$ copies of $\cn'$, and satisfies
\[
Pr_{\bx \sim \mu}\left(\cn(\bx) \neq (p(d)\tbx)^{\bin(c)}\right)
\leq \sum_{i \in [d]} Pr_{x_i \sim \mu_i}\left(\cn'(x_i) \neq (p(d)\tx_i)^{\bin(c)}\right)
\leq \frac{\delta'}{d} \cdot d = \delta'~.
\]

For $j \in [c]$ let $I_j \subseteq \{-R p(d), \ldots, R p(d)\}$ be the integers such that the $j$-th bit in their binary representation is $1$.
Hence, given $x_i$, the network $\cn'$ should output in the $j$-th output $\onefunc_{I_j}(p(d)\tx_i)$.

By Lemma~\ref{lemma:from bounded conditional to marginal}, there is $\Delta = \frac{1}{\poly(d)}$ such that for every $i \in [d]$ and every $t \in \reals$ we have
\begin{equation}
\label{eq:bounded marginal}
Pr_{\bx \sim \mu}\left(x_i \in \left[t-\frac{\Delta}{p(d)},t+\frac{\Delta}{p(d)}\right] \right) \leq \frac{\delta'}{(2Rp(d)+2)d}~.
\end{equation}

For an integer $-Rp(d) \leq l \leq Rp(d)$, let $g_l:\reals \rightarrow \reals$ be such that
\[
g_l(t) = \left[\frac{1}{\Delta}\left(t-l+\frac{1}{2}\right)\right]_+ - \left[\frac{1}{\Delta}\left(t-l+\frac{1}{2}-\Delta\right)\right]_+~.
\]
Note that $g_l(t)=0$ if $t \leq l-\frac{1}{2}$, and that $g_l(t)=1$ if $t \geq l-\frac{1}{2}+\Delta$.
Let $g'_l(t) = g_{l}(t)-g_{l+1}(t)$.
Note that $g'_l(t)=0$ if $t \leq l-\frac{1}{2}$ or $t \geq l+\frac{1}{2}+\Delta$, and that $g'_l(t)=1$ if $l-\frac{1}{2}+\Delta \leq t \leq l+\frac{1}{2}$.

Let $h_j(t) = \sum_{l \in I_j}g'_l(t)$.
Note that for every $l \in \{-R p(d), \ldots, R p(d)\}$ and $l-\frac{1}{2}+\Delta \leq t \leq l+\frac{1}{2}$ we have $h_j(t)=1$ if $l \in I_j$ and $h_j(t)=0$ otherwise.
Hence, for $p(d)x_i \in [-Rp(d)-\frac{1}{2},Rp(d)+\frac{1}{2}]$, if $|p(d)x_i - p(d)\tx_i| \leq \frac{1}{2}-\Delta$ then $h_j(p(d)x_i)=\onefunc_{I_j}(p(d)\tx_i)$.
For $p(d)x_i \leq -Rp(d)-\frac{1}{2}-\Delta$ and for $p(d)x_i \geq Rp(d)+\frac{1}{2} + \Delta$, we have $\tx_i=0$ and $h_j(p(d)x_i) = 0 = \onefunc_{I_j}(0) = \onefunc_{I_j}(p(d)\tx_i)$.
Therefore, if $h_j(p(d)x_i) \neq \onefunc_{I_j}(p(d)\tx_i)$ then $p(d)x_i \in [l-\frac{1}{2}-\Delta,l-\frac{1}{2}+\Delta]$ for some integer $-Rp(d) \leq l \leq Rp(d)+1$.

Let $\cn'$ be such that $\cn'(x_i)=\left(h_1(p(d)x_i),\ldots,h_{c}(p(d)x_i)\right)$.
Note that $\cn'$ can be implemented by a depth-$2$ neural network.

Now,
\begin{align*}
Pr_{x_i \sim \mu_i}&\left(\cn'(x_i) \neq (p(d)\tx_i)^{\bin(c)} \right)
= Pr_{x_i \sim \mu_i}\left(\exists j \in [c] \text{\ s.t.\ } h_j(p(d)x_i) \neq (p(d)\tx_i)^{\bin(c)}_j \right)
\\
&=Pr_{x_i \sim \mu_i}\left(\exists j \in [c] \text{\ s.t.\ } h_j(p(d)x_i) \neq \onefunc_{I_j}(p(d)\tx_i) \right)
\\
&\leq Pr_{x_i \sim \mu_i}\left( p(d)x_i \in \left[l-\frac{1}{2}-\Delta,l-\frac{1}{2}+\Delta\right], -Rp(d) \leq l \leq Rp(d)+1 \right)
\\
&\leq \sum_{-Rp(d) \leq l \leq Rp(d)+1} Pr_{x_i \sim \mu_i}\left(x_i \in \left[\frac{l}{p(d)}-\frac{1}{2p(d)}-\frac{\Delta}{p(d)},\frac{l}{p(d)}-\frac{1}{2p(d)}+\frac{\Delta}{p(d)}\right] \right)
\\
&\stackrel{(Eq.~\ref{eq:bounded marginal})}{\leq} \left(2Rp(d)+2\right) \cdot \frac{\delta'}{(2Rp(d)+2)d}
= \frac{\delta'}{d}~.
\end{align*}
\end{proof}

We now show that $\tbx \mapsto N(\tbx)$ for our network $N$ can be computed approximately by a depth-$k$ network $N''$ whose weights and biases are at most $2^{\poly(d)}$, and have a binary representation with $\poly(d)$ bits. The network $N''$ will be useful later in order to simulate such a computation with arithmetic operations on binary vectors.

\begin{lemma} (\cite{maass1997bounds})
\label{lemma:from maass}
Consider a system $A\bx \leq \bb$ of arbitrary finite number of linear inequalities in $l$ variables. Assume that all entries in $A$ and $\bb$ are integers of absolute value at most $a$. If this system has a solution in $\reals^l$, then it has a solution of the form $\left(\frac{s_1}{t},\ldots,\frac{s_l}{t}\right)$, where $s_1,\ldots,s_l,t$ are integers of absolute value at most $(2l+1)!a^{2l+1}$.
\end{lemma}

\begin{lemma}
\label{lemma:to digital}
Let $p'(d) = \poly(d)$. There is a $\poly(d)$-sized neural network $N''$ of depth $k$ such that for every $\tbx \in \ci^d$ we have:
\begin{itemize}
\item If $N(\tbx) \in [-B,B]$ then $|N''(\tbx)-N(\tbx)| \leq \frac{1}{p'(d)}$.
\item If $N(\tbx) > B$ then $N''(\tbx) \geq B$.
\item If $N(\tbx) < -B$ then $N''(\tbx) \leq -B$.
\end{itemize}
Moreover, $N''$ satisfies the following:
\begin{itemize}
\item There is a positive integer $t \leq 2^{\poly(d)}$ such that all weights and biases are in $Q_t=\{\frac{s}{t}: |s| \leq 2^{\poly(d)}, s \in \integers\}$.
\item The weights in layers $2,\ldots,k$ are all in $\{-1,1\}$.
\end{itemize}
\end{lemma}

\begin{proof}
In \cite{maass1997bounds} it is shown that a similar property holds for the case where the output neuron has sign activation, namely, where the output of the network is Boolean. We extend this result to real-valued functions.

We construct $N''$ in three steps. First, we transform $N$ into a network $N_1$ of depth $k$ where the fan-out of each hidden neuron is $1$, such that for every $\bx \in \reals^d$ we have $N_1(\bx)=N(\bx)$.
Then, we transform $N_1$ into a network $N_2$ of depth $k$ where the weights in layers $2,\ldots,k$ are all in $\{-1,1\}$, such that for every $\bx \in \reals^d$ we have $N_2(\bx)=N_1(\bx)$.
Finally, we show that $N_2$ can be transformed to a network $N''$ that satisfies the requirements (in particular, with exponentially-bounded weights and biases). The last stage is the most delicate one, and can be roughly described as follows: We create a huge set of linear inequalities, which encodes the requirement that the weights and biases of each neuron in $N_2$ produce the appropriate outputs, separately for each and every possible input $\tilde{\bx}$ from our grid $\ci^d$ (up to polynomially small error). Moreover,
it can be shown that the size of the elements in our linear inequalities is $\poly(d)$. Hence, invoking Lemma \ref{lemma:from maass}, we get that there is a solution to the linear system (namely, a set of weights and biases) which approximate $N_2$, yet have only $2^{\poly(d)}$-sized entries.


We now turn to the formal proof. First, the network $N_1$ is obtained by proceeding inductively from the output neuron towards the input neurons. Each hidden neuron with fan-out $c>1$ is duplicated $c$ times. Let $l_i,l'_i$ be the number of hidden neurons in the $i$-th layer of $N$ and $N_1$ respectively. Note that $l'_i \leq l_i \cdot l'_{i+1}$. Since $k$ is constant and $l_i = \poly(d)$ then the size of $N_1$ is also $\poly(d)$.

In order to construct $N_2$, we, again, proceed inductively from the output neuron $n_{\text{out}}$ of $N_1$ towards the input neurons. Let $w_1,\ldots,w_l$ be the weights of the output neuron and let $n_1,\ldots,n_l$ be the corresponding hidden neurons. That is, for each $i \in [l]$ there is an edge with weight $w_i \neq 0$ between $n_i$ and $n_{\text{out}}$. Now, for each $i \in [l]$, we replace the weight $w_i$ of the edge $(n_i,n_{\text{out}})$ by $\frac{w_i}{|w_i|}$, and multiply the weights and bias of $n_i$ by $|w_i|$. Note that now the multiplication by $|w_i|$ is done before $n_i$ instead of after it, but $n_{\text{out}}$ still receives the same input as in $N_1$. Since the fan-out of every hidden neuron in $N_1$ is $1$, we can repeat the same operation also in the predecessors of $n_1,\ldots,n_l$, and continue until the first hidden layer. Hence, we obtain a network $N_2$ where the weights in layers $2,\ldots,k$ are all in $\{-1,1\}$.

We now show that $N_2$ can be transformed to a network $N''$ that satisfies the requirements.
Let $l_1$ be the number of neurons in the first hidden layer of $N_2$, let $m_w=d \cdot l_1$ be the number of weights in the first layer (including $0$ weights), and let $m_b$ the number of hidden neurons in $N_2$, that is, the number of biases in $N_2$. Let $m=m_w + m_b$. For each $i \in [l_1]$ we denote by $\bw_i \in \reals^d$ the weights of the $i$-th neuron in the first hidden layer in $N_2$, and for each hidden neuron $n$ in $N_2$ we denote by $b_n$ the bias of $n$.
We define a linear system $A\bz \leq \bc$ where the variables $\bz \in \reals^m$ correspond to the weights of the first layer and the biases in $N_2$.
We denote by $\bz_i$ the $d$ variables in $\bz$ that correspond to $\bw_i$, and by $z_n$ the variable in $\bz$ that corresponds to $b_n$.
Note that each assignment to the variables $\bz$ induces a neural network $N_2^\bz$ where the weights in the first layer and the biases in $N_2$ are replaced by the corresponding variables.

For each $\tbx \in \ci^d$ we place in the system $A\bz \leq \bc$ an inequality for each hidden neuron in $N_2$, and either one or two inequalities for the output neuron. These inequalities are defined by induction on the depth of the neuron.
If $n_i$ is the $i$-th neuron in the first hidden layer and its input in the computation of $N_2(\tbx)$ satisfies $\inner{\tbx,\bw_i}+b_{n_i} \geq 0$, then we add the inequality $\inner{\tbx,\bz_i}+z_{n_i} \geq 0$ to the system. Otherwise, we add the inequality $\inner{\tbx,\bz_i}+z_{n_i} \leq 0$. Note that the variables in the inequality are $\bz_i,z_{n_i}$, and that $\tbx$ is a constant.
Let $S_1 \subseteq \{n_i: i \in [l_1]\}$ be the neurons in the first hidden layer where $\inner{\tbx,\bw_i}+b_{n_i} \geq 0$, that is, the neurons where the ReLU is active in the computation $N_2(\tbx)$.
Now, the input for each neuron $n'$ in the second hidden layer in the computation $N_2(\tbx)$, is of the form $I(n')=\sum_{n_i \in S_1}a_i (\inner{\tbx,\bw_i}+b_{n_i}) + b_{n'}$ where $a_i \in \{-1,0,1\}$ is the weight of the edge $(n_i,n')$ in $N_2$. Let $I'(n')=\sum_{n_i \in S_1}a_i (\inner{\tbx,\bz_i}+z_{n_i}) + z_{n'}$. If $I(n') \geq 0$ then we add the inequality $I'(n') \geq 0$, and otherwise we add $I'(n') \leq 0$. Note that the variables in the inequality are $\bz_i,z_{n_i},z_{n'}$ (for the appropriate indices $i$) and that $\tbx,a_i$ are constants. Thus, this inequality is linear.

We denote by $S_2$ the set of neurons in the second hidden layer where the ReLU is active in the computation $N_2(\tbx)$, and for each neuron $n''$ in the third hidden layer we define $I(n'')$ and $I'(n'')$ and add a linear inequality analogously.
We continue until we reach the output neuron $n_{\text{out}}$.
Let $I(n_{\text{out}})$ be the input to $n_{\text{out}}$ in the computation $N_2(\tbx)$, and let $I'(n_{\text{out}})$ be the corresponding linear expression, where the variables are $\bz$ and the constants are $\tbx$ and the weights in layers $2,\ldots,k$ (which are all in $\{-1,0,1\}$). Note that $I(n_{\text{out}})=N_2(\tbx)=N(\tbx)$.
If $N(\tbx) \in [-B,B]$, then let $-B p'(d) \leq j \leq B p'(d) - 1$ be an integer such that $\frac{j}{p'(d)} \leq I(n_{\text{out}}) \leq \frac{j+1}{p'(d)}$. Now, we add the two inequalities $\frac{j}{p'(d)} \leq I'(n_{\text{out}}) \leq \frac{j+1}{p'(d)}$, where $j,p'(d)$ are constants.
If $N(\tbx) > B$, then we add the inequality $I'(n_{\text{out}}) \geq B$, and if $N(\tbx) < -B$, then we add the inequality $I'(n_{\text{out}}) \leq -B$.

Note that if $\bz$ satisfies all the inequalities $A\bz \leq \bc$, then for each neuron $n$, the expression $I'(n)$ is consistent with the set of active ReLUs according to the inequalities of the previous layers. Therefore, the input to $n$ in the computation $N_2^\bz(\tbx)$ is $I'(n)$.
Hence, for such $\bz$ we have for every $\tbx \in \ci^d$ that if $N_2(\tbx) \in [-B,B]$ then $|N_2^\bz(\tbx)-N_2(\tbx)| \leq \frac{1}{p'(d)}$, if $N_2(\tbx) > B$ then $N_2^\bz(\tbx) \geq B$, and if $N_2(\tbx) < -B$ then $N_2^\bz(\tbx) \leq -B$.
Note that $A\bz \leq \bc$ has a solution in $\reals^m$, since the weights and biases in $N_2$ satisfy all the inequalities.
The entries in $A,\bc$ are either integers with absolute value at most $\poly(d)$, or of the form $q \cdot \tx_i = \frac{q'}{p(d)}$ or $\frac{q}{p'(d)}$ where $q,q'$ are integers with absolute values at most $\poly(d)$. Therefore, by Lemma~\ref{lemma:from maass}, there is an integer $a=\poly(d)$ such that the linear system $(p(d)p'(d)A)\bz \leq p(d)p'(d)\bc$ has a solution $\bz=\left(\frac{s_1}{t},\ldots,\frac{s_m}{t}\right)$, where $s_1,\ldots,s_m,t$ are integers of absolute value at most $(2m+1)!a^{2m+1} \leq 2^{\poly(d)}$. Hence, the network $N''=N_2^\bz$ satisfies the requirements.
\end{proof}

Let $N''$ be the network from Lemma~\ref{lemma:to digital} with $p'(d)=\frac{\sqrt{50}}{\epsilon}$.
The following lemma follows easily.
\begin{lemma}
\label{lemma:N'' approximates N'}
For every $\tbx \in \ci^d$ we have
\[
\left|[N''(\tbx)]_{[-B,B]} - N'(\tbx)\right| \leq \frac{1}{p'(d)}~.
\]
\end{lemma}
\begin{proof}
\begin{itemize}
\item If $N(\tbx) \in [-B,B]$ then $|N''(\tbx)-N(\tbx)| \leq \frac{1}{p'(d)}$ and we have
\[
\left|[N''(\tbx)]_{[-B,B]} - N'(\tbx)\right|
\leq \left|N''(\tbx) - N'(\tbx)\right|
= \left|N''(\tbx) - N(\tbx)\right|
\leq \frac{1}{p'(d)}~.
\]
\item If $N(\tbx) > B$ then $N''(\tbx) \geq B$, and therefore
\[
\left|[N''(\tbx)]_{[-B,B]} - N'(\tbx)\right|
= |B - B|
= 0~.
\]
\item If $N(\tbx) < -B$ then $N''(\tbx) \leq -B$, and therefore
\[
\left|[N''(\tbx)]_{[-B,B]} - N'(\tbx)\right|
= |-B - (-B)|
= 0~.
\]
\end{itemize}
\end{proof}

The weights and biases in $N''$ might be exponential, but they have a binary representation with $\poly(d)$ bits. This property enables us to simulate $[N''(\tbx)]_{[-B,B]}$ using arithmetic operations on binary vectors.

We now show how to simulate $[N''(\tbx)]_{[-B,B]}$ using binary operations.
Recall that the input $\tbx$ to $N''$ is such that every component $\tx_i$ is of the form $\frac{q_i}{p(d)}$ for some integer $q_i$ with absolute value at most $\poly(d)$. We will represent each component in the input by the binary representation of the integer $p(d)\tx_i$. It implies that while simulating $N''$, we should replace each weight $w$ in the first layer of $N''$ with $w' = \frac{w}{p(d)}$. Then, $w \cdot \tx_i = w' \cdot (p(d)\tx_i)$.
Recall that the network $N''$ is such that all weights in layers $2,\ldots,k$ in $N''$ are in $\{-1,1\}$ and all weights in the first layer and biases are of the form $\frac{s_i}{t}$ for some positive integer $t \leq 2^{\poly(d)}$, and integers $s_i$ with $|s_i| \leq 2^{\poly(d)}$.
We represent each number of the form $\frac{v}{t}$ by the binary representation of $v$.
Since for all weights and biases $\frac{s_i}{t}$ in $N''$ we can multiply both $t$ and $s_i$ by $p(d)$, we can assume w.l.o.g. that $p(d) \divides s_i$ and $p(d) \divides t$. Then, for each weight $w=\frac{s_i}{t}$ in the first layer of $N''$, we represent $w' = \frac{w}{p(d)} = \frac{s_i}{t \cdot p(d)}$ by the binary representation of the integer $\frac{s_i}{p(d)}$.

Since the input to a neuron in the first hidden layer of $N''$ is a sum of the form $I = \sum_{i \in [d]}w_i \tx_i + b = \sum_{i \in [d]}w'_i (p(d)\tx_i) + b$, then in order to simulate it we need to compute multiplications and additions of binary vectors. Note that $p(d)\tx_i$ are integers, $w'_i$ are represented by the binary representation of the integers $q_i$ such that $w'_i = \frac{q_i}{t}$, and $b$ is represented by the binary representation of the integer $q$ such that $b = \frac{q}{t}$. Then, $I$ is also of the form $\frac{v}{t}$ for an integer $v$ with $|v| \leq 2^{\poly(d)}$, and therefore it can be represented by the binary representation of $v$.
Since the biases in $N''$ are of the form $\frac{s_i}{t}$ for integers $s_i$, and the weights in layers $2,\ldots,k$ are in $\{-1,1\}$, then in the computation $N''(\tbx)$ all values, namely, inputs to neurons in all layers, are of the form $\frac{v}{t}$ where $v$ is an integer with $|v| \leq 2^{\poly(d)}$. That is, a binary representation of $v$ requires $\poly(d)$ bits.
Thus, since all values have $t$ in the denominator, then we ignore it and work only with the numerator.

Let $C' = \poly(d)$ be such that for all $\tbx \in \ci^d$, all inputs to neurons in the computation $N''(\tbx)$ are of the form $\frac{v}{t}$ where $v$ is an integer with absolute value at most $2^{C'}$. Namely, all values in the computation can be represented by $C'$ bits. Let $C = \poly(d)$ be such that every integer $v$ of absolute value at most $2^{C'} + Bt$ has a binary representation with $C$ bits. Also, assume that $C \geq \log(2Rp(d) +1)$.
Such $C$ will be sufficiently large in order to represent all inputs $p(d)\tx_i$ and all values in our simulation of $[N''(\tbx)]_{[-B,B]}$.

We now show how to simulate $p(d)\tbx \mapsto [N''(\tbx)]_{[-B,B]} + B$ with a threshold circuit.

\begin{lemma}
\label{lemma:simulation with TC}
There is a threshold circuit $T$ of depth $3k+1$, width $\poly(d)$, and $\poly(d)$-bounded weights, whose inputs are the $C$-bits binary representations of:
\begin{itemize}
\item $p(d)\tx_i$ for every $i \in [d]$.
\item $\frac{s_i}{p(d)}$ and $\frac{-s_i}{p(d)}$ for every weight $\frac{s_i}{t}$ in the first layer of $N''$.
\end{itemize}
And its outputs are:
\begin{itemize}
\item The $C$-bits binary representation of $v$ such that:
    \begin{itemize}
    \item If $N''(\tbx) \in [-B,B]$ then $\frac{v}{t} = N''(\tbx) + B$.
    \item Otherwise $v=0$.
    \end{itemize}
\item A bit $c$ such that $c=1$ iff $N''(\tbx) > B$.
\end{itemize}
\end{lemma}
\begin{proof}
In order to simulate the first layer of $N''$, we first need to compute a sum of the form $\sum_{i \in [d]}w_i \cdot z_i$ where $w_i = \frac{s_i}{p(d)}$ and $z_i=p(d)\tx_i$ are the inputs and are given in a binary representation. Hence, we are required to perform binary multiplications and then binary iterated addition, namely, addition of multiple numbers that are given by binary vectors.
Binary iterated addition can be done by a depth-$2$ threshold circuit with polynomially-bounded weights and polynomial width, and binary multiplication can be done by a depth-$3$ threshold circuit with polynomially-bounded weights and polynomial width (\cite{siu1994optimal}). The depth-$3$ circuit for multiplication shown in \cite{siu1994optimal} first computes the partial products and then uses the depth-$2$ threshold circuit for iterated addition in order to compute their sum. They show it for a multiplication of two $n$-bit numbers that results in a $2n$-bit number. The same method can be used also in our case for a multiplication of two $C$-bit numbers that results in a $C$-bit number, since $C$ was chosen such that we are guaranteed that there is no overflow. Also, in two's complement representation, multiplication and addition of signed numbers can be done similarly to the unsigned case.
In our case, we need to compute multiplication and then iterated addition. Hence, instead of using a depth-$5$ threshold circuit that computes multiplications and then computes the iterated addition, we can use a depth-$3$ threshold circuit that first computes all partial products for all multiplications, and then computes a single iterated addition.

Since the hidden neurons in $N''$ have biases, we need to simulate sums of the form $b+\sum_{i \in [d]} w_i \cdot z_i$. Hence, the binary iterated addition should also include $b$. Therefore, the bias $b$ is hardwired into the circuit $T$. That is, for every bias $b=\frac{v}{t}$, we add $C$ gates to the first hidden layer with fan-in $0$ and with biases in $\{0,1\}$ that correspond to the binary representation of $v$.

Simulating the ReLUs of the first hidden layer in $N''$ can be done as follows.
Let $v$ be an integer and let $v^{\bin(C)} \in \{0,1\}^C$ be its binary representation.
Recall that in the two's complement representation the most significant bit (MSB) is $1$ iff the number is negative.
Now, we reduce the value of the MSB, namely $v^{\bin(C)}_C$, from all other $C-1$ bits $v^{\bin(C)}_1,\ldots,v^{\bin(C)}_{C-1}$. Thus, we transform $v^{\bin(C)}$ to $\left(\sign(v^{\bin(C)}_1-v^{\bin(C)}_C),\ldots,\sign(v^{\bin(C)}_{C-1}-v^{\bin(C)}_C),0\right)$. Now, if $v<0$, that is $v^{\bin(C)}_C=1$, then we obtain a binary vector whose bits are all $0$. If $v \geq 0$ then $v^{\bin(C)}_C=0$ and therefore $v^{\bin(C)}$ is not changed. Thus, simulating a ReLU of $N''$ requires one additional layer in the threshold circuit.
Overall, the output of the first hidden layer of $N''$ can be computed by a depth-$4$ threshold circuit.

Now, the weights in layers $2,\ldots,k$ in $N''$ are in $\{-1,1\}$. Note that simulating multiplication by a threshold circuit, as discussed above, requires $3$ layers. However, we need to compute values of the form $b + \sum_i a_i \cdot z_i$ where $a_i \in \{-1,1\}$, and $z_i,b$ are given by binary vectors. In order to avoid multiplication, we keep both the values of the computation $N''(\tbx)$ in each layer, and their negations. That is, the circuit $T$ keeps both the binary representation of $z_i$ and the binary representation of $-z_i$, and then simulating each layer can be done by iterated addition, without binary multiplication. Keeping both $z_i$ and $-z_i$ in each layer is done as follows. When $T$ simulates the first layer of $N''$, it computes values of the form $z=b+\sum_{i \in [d]} w_i \cdot (p(d)\tx_i)$, and in parallel it should also compute $-z=-b+\sum_{i \in [d]} (-w_i) \cdot (p(d)\tx_i)$. Note that both $w_i$ and $-w_i$ are given as inputs to $T$, and that the binary representation of $v$ such that $-b=\frac{v}{t}$ can be hardwired into $T$, similarly to the case of $b$.
Then, when simulating a ReLU of $N''$, it reduces the MSB of $z$ also from all bits of the binary representation of $-z$. Thus, if $z < 0$ then both $z$ and $-z$ become $0$. Now, computing $z'= b' + \sum_i a_i \cdot z_i$ where $a_i \in \{-1,1\}$ and $z_i,b'$ are binary numbers, can be done by iterated addition, and also computing $-z'= -b' + \sum_i -a_i \cdot z_i$ can be done by iterated addition.
Note that the binary representations of $\pm v$ such that $b'=\frac{v}{t}$ are also hardwired into $T$.
Since iterated addition can be implemented by a threshold circuit of depth $2$, the sum $z'= b' + \sum_i a_i \cdot z_i$ can be implemented by $2$ layers in $T$, and then implementing $[z']_+$, requires one more layer as discussed above.
Thus, each of the layers $2,\ldots,k-1$ in $N''$ requires $3$ layers in $T$.

Let $N''(\tbx)=\frac{v}{t}$. When simulating the final layer of $N''$, we also add (as a part of the iterated addition) the hardwired binary representation of $Bt$. That is, instead of computing the binary representation of $v$, we compute the binary representation of $v' = v+Bt$. We also compute the binary representation of $v'' = -v + Bt$. Note that $\frac{v'}{t} = N''(\tbx) + B$ and $\frac{v''}{t} = -N''(\tbx) + B$.
Now, the bit $c$ that $T$ should output is the MSB of $v''$, since $v''$ is negative iff $N''(\tbx) > B$.
The $C$-bits binary vector that $T$ outputs is obtained from $v',v''$ by adding one final layer as follows.
Let $\msb(v')$ and $\msb(v'')$ be the MSBs of $v',v''$. In the final layer we reduce $\msb(v')+\msb(v'')$ from all bits of $v'$. That is, if either $v'$ or $v''$ are negative, then we output $0$, and otherwise we output $v'$.
Now, if $N''(\tbx) \in [-B,B]$ then $v',v'' \in [0,2Bt]$, and we output $v'$, which corresponds to $N''(\tbx) + B$.
If $N''(\tbx) < -B$ then $\frac{v'}{t} = N''(\tbx) + B < 0$, and therefore $\msb(v')=1$, and we output $0$.
If $N''(\tbx) > B$ then $\frac{v''}{t} = -N''(\tbx) + B <0$, and therefore $\msb(v'')=1$, and we output $0$.
Thus, simulating the final layer of $N''$ requires $3$ layers in $T$: $2$ layers for the iterated addition, and one layer for transforming $v',v''$ to the required output.

Finally, the depth of $T$ is $3k+1$ since simulating the first layer of $N''$ requires $4$ layers in $T$, and each additional layer in $N''$ required $3$ layers in $T$.
\end{proof}

The following simple lemma shows that threshold circuits can be transformed to neural networks.

\begin{lemma}
\label{lemma:from TC to NN}
Let $T$ be a threshold circuit with $d$ inputs, $q$ outputs, depth $m$ and width $\poly(d)$. There is a neural network $\cn$ with $q$ outputs, depth $m+1$ and width $\poly(d)$, such that for every $\bx \in \{0,1\}^d$ we have $\cn(\bx) = T(\bx)$. If $T$ has $\poly(d)$-bounded weights then $\cn$ also has $\poly(d)$-bounded weights.
Moreover, for every input $\bx \in \reals^d$ the outputs of $\cn$ are in $[0,1]$.
\end{lemma}
\begin{proof}
Let $g$ be a gate in $T$, and let $\bw \in \integers^l$ and $b \in \integers$ be its weights and bias. Let $n_1$ be a neuron with weights $\bw$ and bias $b$, and let $n_2$ be a neuron with weights $\bw$ and bias $b-1$. Let $\by \in \{0,1\}^l$. Since $(\inner{\bw,\by}+b) \in \integers$, we have $[\inner{\bw,\by}+b]_+ - [\inner{\bw,\by}+b-1]_+ = \sign(\inner{\bw,\by} + b)$. Hence, the gate $g$ can be replaced by the neurons $n_1,n_2$. We replace all gates in $T$ by neurons and obtain a network $\cn$. Since each output gate of $T$ is also replaced by two neurons, $\cn$ has $m+1$ layers.
Since for every $\bx \in \reals^d$, weight vector $\bw$ and bias $b$ we have $[\inner{\bw,\bx}+b]_+ - [\inner{\bw,\bx}+b-1]_+ \in [0,1]$ then for every input $\bx \in \reals^d$ the outputs of $\cn(\bx)$ are in $[0,1]$.
\end{proof}

We are now ready to construct the network $\hat{N}$.
Let $\delta'=\frac{\epsilon^2}{50 \cdot 36B^2}$.
The network $\hat{N}$ is such that w.p. at least $1-\delta'$ we have $\hat{N}(\bx) = [N''(\tbx)]_{[-B,B]}$. It consists of three parts.

First, $\hat{N}$ transforms w.p. $\geq 1-\delta'$ the input $\bx$ to the $(C \cdot d)$-bits binary representation of $p(d)\tbx$. By Lemma~\ref{lemma:from bx to binary}, it can be done with a $2$-layers neural network $\cn_1$.

Second, let $T$ be the threshold circuit from Lemma~\ref{lemma:simulation with TC}. By Lemma~\ref{lemma:from TC to NN}, $T$ can be implemented by a neural network $\cn_2$ of depth $3k+2$. Note that the input to $\cn_2$ has two parts:
\begin{enumerate}
\item The $(C \cdot d)$-bits binary representation of $p(d)\tbx$. This is the output of $\cn_1$.
\item The binary representations of $\frac{\pm s_i}{p(d)}$ for every weight $\frac{s_i}{t}$ in the first layer of $N''$. This is hardwired into $\hat{N}$ by hidden neurons with fan-in $0$ and appropriate biases in $\{0,1\}$.
\end{enumerate}
Thus, using $\cn_2$ the network $\hat{N}$ transforms the binary representation of $p(d)\tbx$ to the output $(v^{\bin(C)},c)$ of $T$.

Third, $\hat{N}$ transforms $(v^{\bin(C)},c)$ to $[N''(\tbx)]_{[-B,B]}$ as follows.
Let $v$ be the integer that corresponds to the binary vector $v^{\bin(C)}$.
The properties of $v$ and $c$ from Lemma~\ref{lemma:simulation with TC}, imply that $[N''(\tbx)]_{[-B,B]} = \frac{v}{t} + c \cdot 2B - B$, since we have:
\begin{itemize}
\item If $N''(\tbx) \in [-B,B]$ then $\frac{v}{t} = N''(\tbx) + B$ and $c=0$.
\item If $N''(\tbx) < -B$ then $v=0$ and $c = 0$.
\item If $N''(\tbx) > B$ then $v=0$ and $c = 1$.
\end{itemize}
Hence, we need to transform $(v^{\bin(C)},c)$ to the real number $\frac{v}{t} + c \cdot 2B - B$.
Note that $v \geq 0$, and therefore we have
\[
\frac{v}{t} = \sum_{i \in [C-1]}v^{\bin(C)}_i \cdot \frac{2^{i-1}}{t}~.
\]
Also, note that $\frac{v}{t} \in [0,2B]$, and therefore for every $i \in [C-1]$ we have $v^{\bin(C)}_i \cdot \frac{2^{i-1}}{t} \leq 2B$. Hence, we can ignore every $i > \log(2Bt)+1$.
Thus,
\begin{equation}
\label{eq:sum for vt}
\frac{v}{t} = \sum_{i \in [\log(2Bt)+1]}v^{\bin(C)}_i \cdot \frac{2^{i-1}}{t}~.
\end{equation}
Since for $i \in [\log(2Bt)+1]$ we have $\frac{2^{i-1}}{t} \leq 2B$, then in the above computation of $\frac{v}{t}$ the weights are positive numbers smaller or equal to $2B$.
Thus, we can transform $(v^{\bin(C)},c)$ to $\frac{v}{t} + c \cdot 2B - B$ in one layer with $\poly(d)$-bounded weights. In order to avoid bias in the output neuron, the additive term $-B$ is hardwired into $\hat{N}$ by adding a hidden neuron with fan-in $0$ and bias $1$ that is connected to the output neuron with weight $-B$.

Since the final layers of $\cn_1$ and $\cn_2$ do not have activations and can be combined with the next layers, and since the third part of $\hat{N}$ is a sum, then the depth of $\hat{N}$ is $3k+3$.

Thus, we have w.p. at least $1-\delta'$ that $\hat{N}(\bx) = [N''(\tbx)]_{[-B,B]}$.
By Lemma~\ref{lemma:N'' approximates N'}, it implies that w.p. at least $1-\delta'$ we have
\begin{equation}
\label{eq:dist hatN tildeN}
\left|\hat{N}(\bx) - \tilde{N}(\bx)\right| \leq \frac{1}{p'(d)}~.
\end{equation}
However, it is possible (w.p. at most $\delta'$) that $\cn_1$ fails to transform $\bx$ to the binary representation of $p(d)\tbx$, and therefore the above inequality does not hold. Still, even in this case we can bound the output of $\hat{N}$ as follows.
If $\cn_1$ fails to transform $\bx$ to $p(d)\tbx$, then the input to $\cn_2$ may contains values other than $\{0,1\}$. However, by Lemma~\ref{lemma:from TC to NN}, the network $\cn_2$ outputs $c$ and $v^{\bin(C)}$ such that each component is in $[0,1]$. Now, when transforming $(v^{\bin(C)},c)$ to $\frac{v}{t} + c \cdot 2B - B$ in the final layer of $\hat{N}$, we compute $\frac{v}{t}$ by the sum in Eq.~\ref{eq:sum for vt}.
Since $v^{\bin(C)}_i \in [0,1]$ for every $i$, this sum is at least $0$ and at most
\[
\frac{1}{t} \cdot 2^{\log(2Bt)+1}
= 4B~.
\]
Therefore, the output of $\hat{N}$ is at most $4B + 1 \cdot 2B - B = 5B$, and at least $0 + 0 \cdot 2B - B = -B$. Thus, for every $\bx$ we have $\hat{N}(\bx) \in [-B,5B]$. Since for every $\bx$ we have $\tilde{N}(\bx) \in [-B,B]$, then we have
\[
\left|\hat{N}(\bx) - \tilde{N}(\bx)\right| \leq 6B~.
\]

Combining the above with Eq.~\ref{eq:dist hatN tildeN} and plugging in $\delta'=\frac{\epsilon^2}{50 \cdot 36B^2}$ and $p'(d)=\frac{\sqrt{50}}{\epsilon}$, we have
\[
\E_{\bx \sim \mu}(\hat{N}(\bx)-\tilde{N}(\bx))^2
\leq (1-\delta')\left(\frac{1}{p'(d)}\right)^2 + \delta' \cdot (6B)^2
= (1-\delta')\frac{\epsilon^2}{50} + \frac{\epsilon^2}{50 \cdot 36B^2} \cdot 36B^2
\leq \left(\frac{\epsilon}{5}\right)^2~.
\]
Therefore $\snorm{\hat{N}-\tilde{N}}_{L_2(\mu)} \leq \frac{\epsilon}{5}$ as required.

\subsection{Proof of Theorem~\ref{thm:poly weights discrete}}

The proof follows the same ideas as the proof of Theorem~\ref{thm:poly weights}, but is simpler.
Consider the functions $f'$ and $N'$ that are defined in the proof of Theorem~\ref{thm:poly weights}. For every $\bx \in \ci^d$ we denote $\tilde{\bx}=\bx$, and $\tilde{N}(\bx) = N'(\tbx) = N'(\bx)$.
Now, from the same arguments as in the proof of Theorem~\ref{thm:poly weights}, it follows that we can bound $\snorm{f'-f}_{L_2(\cd)}$ and $\snorm{N'-f'}_{L_2(\cd)}$. Since $\snorm{\tilde{N}-N'}_{L_2(\cd)} = 0$, it remains to show that Lemma~\ref{lemma:triangle4} holds also in this case.

The network $\hat{N}$ will have a similar structure to the one in the proof of Lemma~\ref{lemma:triangle4}.

First, it transforms the input $\bx$ to the binary representation of $p(d)\tilde{\bx} = p(d)\bx$. This transformation is similar to the one from the proof of Lemma~\ref{lemma:from bx to binary}. However, since $\cd$ is such that for every $i \in [d]$ the component $x_i$ is of the form $\frac{j}{p(d)}$ for some integer $j$, then for an appropriate $\Delta$, we have for every integer $l$ that
\[
Pr_{\bx \sim \cd}\left(x_i \in \left[\frac{l}{p(d)}-\frac{1}{2p(d)}-\frac{\Delta}{p(d)}, \frac{l}{p(d)}-\frac{1}{2p(d)}+\frac{\Delta}{p(d)} \right] \right) = 0~.
\]
Hence, there is a depth-$2$ network with $\poly(d)$ width and $\poly(d)$-bounded weights, that transforms $\bx$ to the binary representation of $p(d)\tbx$ and succeeds w.p.~$1$.

Recall that in the proof of Lemma~\ref{lemma:triangle4}, the next parts of $\hat{N}$ transform for every $\tbx$ the binary representation of $p(d)\tbx$ to $[N''(\tbx)]_{[-B,B]}$. Since this transformation is already discrete and does not depend on the input distribution, we can also use it here.
Then, by lemma~\ref{lemma:N'' approximates N'} we have for every $\tbx$ that
\[
\left|[N''(\tbx)]_{[-B,B]} - N'(\tbx)\right| \leq \frac{1}{p'(d)}~.
\]

Thus, for $p'(d) = \frac{5}{\epsilon}$, we obtain a network $\hat{N}$ such that w.p.~$1$ we have $\left|\hat{N}(\bx) - \tilde{N}(\bx)\right| \leq \frac{\epsilon}{5}$,
and therefore $\snorm{\hat{N}-\tilde{N}}_{L_2(\cd)} \leq \frac{\epsilon}{5}$.

\subsection{Proof of Theorem~\ref{thm:separation}}

Let $\epsilon = \frac{1}{\poly(d)}$, and let $N$ be a neural network of depth $k'$ such that $\snorm{N-f}_{L_2(\mu)} \leq \frac{\epsilon}{5}$. In the proof of Theorem~\ref{thm:poly weights} we constructed a network $\hat{N}$ of depth $3k'+3$ such that $\snorm{\hat{N}-f}_{L_2(\mu)} \leq \epsilon$.
The network $\hat{N}$ is such that in the first two layers the input $\bx$ is transformed w.h.p. to the binary representation of $p(d)\tbx$. This transformation requires two layers, denoted by $\cn_1$. Since the second layer in $\cn_1$ does not have activation, it is combined with the next layer in $\hat{N}$.
The next layers in $\hat{N}$, denoted by $\cn_2$, implement a threshold circuit $T$ of depth $3k'+1$ and width $\poly(d)$. The depth of $\cn_2$ is $3k'+2$. Since the final layer of $\cn_2$ does not have activation, it is combined with the next layer in $\hat{N}$.
Finally, the output of $\hat{N}$ is obtained by computing a linear function over the outputs of $\cn_2$.

Let $g:\{0,1\}^{d'} \rightarrow \{0,1\}^{C+1}$ be the function that $T$ computes. Note that $d'=\poly(d)$ and $C = \poly(d)$.
Assume that $g$ can be computed by a threshold circuit $T'$ of depth $k-2$ and width $\poly(d')$.
By Lemma~\ref{lemma:from TC to NN}, the threshold circuit $T'$ can be implemented by a neural network $\cn'_2$ of depth $k-1$ and width $\poly(d')$.
Consider the neural network $\hat{\cn}$ obtained from $\hat{N}$ by replacing $\cn_2$ with $\cn'_2$. The depth of $\hat{\cn}$ is $k$.
The same arguments from the proof of Theorem~\ref{thm:poly weights} for showing that $\snorm{\hat{N}-f}_{L_2(\mu)} \leq \epsilon$ now apply on $\hat{\cn}$, and hence $\snorm{\hat{\cn}-f}_{L_2(\mu)} \leq \epsilon$.
Therefore, $f$ can be approximated by a network of depth $k$, in contradiction to the assumption. Hence the function $g$ cannot be computed by a $\poly(d')$-sized threshold circuit of depth $k-2$.

\subsection{Proof of Theorem~\ref{thm:separation discrete}}

Let $\epsilon = \frac{1}{\poly(d)}$, and let $N$ be a neural network of depth $k'$ such that $\snorm{N-f}_{L_2(\cd)} \leq \frac{\epsilon}{5}$. In the proof of Theorem~\ref{thm:poly weights discrete} we constructed a network $\hat{N}$ of depth $3k'+3$ such that $\snorm{\hat{N}-f}_{L_2(\cd)} \leq \epsilon$.
The structure of the network $\hat{N}$ is similar to the corresponding network from the proof of Theorem~\ref{thm:poly weights}.
Now, the proof follows the same lines as the proof of Theorem~\ref{thm:separation}.

\subsection{Proof of Theorem~\ref{thm:radial}}

Let $f(\bx)=g(\snorm{\bx})$ where $g:\reals \rightarrow \reals$.
Let $\epsilon  = \frac{1}{\poly(d)}$.
By Theorem~\ref{thm:poly weights}, there is a neural network $N$ of a constant depth $k$, width $\poly(d)$, and $\poly(d)$-bounded weights, such that $\E_{\bx \sim \mu}(N(\bx)-f(\bx))^2 \leq \left(\frac{\epsilon}{3}\right)^2$.
Since $N$ has a constant depth, $\poly(d)$ width and $\poly(d)$-bounded weights, then it is $\poly(d)$-Lipschitz. Also, as we show in the proof of Theorem~\ref{thm:poly weights}, the network $N$ is bounded by some $B=\poly(d)$, namely, for every $\bx \in \reals^d$ we have $|N(\bx)| \leq B$.

Let $r=\snorm{\bx}$ and let $\mu_r$ be the distribution of $r$ where $\bx \sim \mu$. Let $U(\bbs^{d-1})$ be the uniform distribution on the unit sphere in $\reals^d$.
Since $\mu$ is radial, we have
\[
\left(\frac{\epsilon}{3}\right)^2 \geq
\E_{\bx \sim \mu}(N(\bx)-f(\bx))^2
= \E_{\bz \sim U(\bbs^{d-1})} \E_{r \sim \mu_r} (N(r\bz)-f(r\bz))^2
= \E_{\bz \sim U(\bbs^{d-1})} \E_{r \sim \mu_r} (N(r\bz)-g(r))^2~.
\]
Therefore, there is some $\bu \in \bbs^{d-1}$ such that $\E_{r \sim \mu_r} (N(r\bu)-g(r))^2 \leq \left(\frac{\epsilon}{3}\right)^2$.
Let $N_\bu:\reals \rightarrow \reals$ be such that $N_\bu(t)=N(t\bu)$. It can be implemented by a network of depth $k$ that is obtained by preceding $N$ with a layer that computes $t \mapsto t\bu$ (and does not have activation). Thus, $\E_{r \sim \mu_r} (N_\bu(r)-g(r))^2 \leq \left(\frac{\epsilon}{3}\right)^2$.
Let $h:\reals^d \rightarrow \reals$ be such that $h(\bx) = N_\bu(\snorm{\bx})$. Note that
\begin{equation}
\label{eq:radial triangle1}
\E_{\bx \sim \mu}(h(\bx)-f(\bx))^2
= \E_{\bx \sim \mu}(N_\bu(\snorm{\bx})-g(\snorm{\bx}))^2
= \E_{r \sim \mu_r}(N_\bu(r)-g(r))^2
\leq \left(\frac{\epsilon}{3}\right)^2~.
\end{equation}

Since $\mu$ has an almost-bounded conditional density, then by Lemma~\ref{lemma:from bounded conditional to marginal}, there is $R_1 = \frac{1}{\poly(d)}$ such that for every $i \in [d]$ we have
\[
Pr_{\bx \sim \mu}\left(x_i \in [-R_1,R_1]\right) \leq \frac{\epsilon^2}{72B^2}~.
\]
Hence,
\[
Pr_{\bx \sim \mu}\left(\snorm{\bx} \leq R_1\right) \leq \frac{\epsilon^2}{72B^2}~.
\]
Also, since $\mu$ has an almost-bounded support, there exists $R_1 < R_2 = \poly(d)$ such that
\[
Pr_{\bx \sim \mu}\left(\snorm{\bx} \geq R_2\right) \leq \frac{\epsilon^2}{72B^2}~.
\]
Thus,
\begin{equation}
\label{eq:bounding r}
Pr_{r \sim \mu_r}(R_1 \leq r \leq R_2) \geq 1 - \frac{\epsilon^2}{36B^2}~.
\end{equation}

Since the network $N$ is bounded by $B$ then $N_\bu$ is also bounded by $B$, namely, for every $t \in \reals$ we have $|N_\bu(t)| \leq B$.
Moreover, since $N$ is $\poly(d)$-Lipschitz, then $N_\bu$ is also $\poly(d)$-Lipschitz.
Let $N'_\bu: \reals \rightarrow \reals$ be such that
\[
	N'_\bu(t) = \begin{cases}
                    0& t \leq \frac{R_1}{2}\\
                    \frac{2 N_\bu(R_1)}{R_1} \cdot t - N_\bu(R_1) &\frac{R_1}{2} < t \leq R_1\\
                    N_\bu(t)& R_1 < t \leq R_2\\
                    -\frac{N_\bu(R_2)}{R_2} \cdot t + 2N_\bu(R_2)&R_2 < t \leq 2R_2\\
                    0& t > 2R_2
                \end{cases}~.
\]
Note that $N'_\bu$ agrees with $N_\bu$ on $[R_1,R_2]$, supported on $\left[\frac{R_1}{2},2R_2\right]$, bounded by $B$, and $\poly(d)$-Lipschitz.
Let $h':\reals^d \rightarrow \reals$ be such that $h'(\bx) = N'_\bu(\snorm{\bx})$.
We have
\[
\E_{\bx \sim \mu}(h'(\bx)-h(\bx))^2
= \E_{\bx \sim \mu}(N'_\bu(\snorm{\bx})-N_\bu(\snorm{\bx}))^2
= \E_{r \sim \mu_r}(N'_\bu(r)-N_\bu(r))^2~.
\]
By Eq.~\ref{eq:bounding r} the functions $N_\bu$ and $N'_\bu$ agree w.p. at least $1-\frac{\epsilon^2}{36B^2}$. Also, since both $N_\bu$ and $N'_\bu$ are bounded by $B$, we have $|N_\bu(r)-N'_\bu(r)| \leq 2B$ for every $r$. Hence, the above is at most
\begin{equation}
\label{eq:radial triangle2}
\frac{\epsilon^2}{36B^2} \cdot (2B)^2 + 0
= \left(\frac{\epsilon}{3}\right)^2~.
\end{equation}

Now, we need the following Lemma.
\begin{lemma} \cite{eldan2016power}
\label{lemma:radial from eldan-shamir}
Let $f:\reals \rightarrow \reals$ be a $\poly(d)$-Lipschitz function supported on [r,R], where $r=\frac{1}{\poly(d)}$ and $R=\poly(d)$. Then, for every $\delta =\frac{1}{\poly(d)}$, there exists a neural network $\cn$ of depth $3$, width $\poly(d)$, and $\poly(d)$-bounded weights, such that
\[
\sup_{\bx \in \reals^d}|\cn(\bx)-f(\snorm{\bx})| \leq \delta~.
\]
\end{lemma}

Since $N'_\bu$ is $\poly(d)$-Lipschitz and supported on $\left[\frac{R_1}{2},2R_2\right]$, then by Lemma~\ref{lemma:radial from eldan-shamir} there exists a network $\cn$ of depth $3$, width $\poly(d)$, and $\poly(d)$-bounded weights , such that
\[
\sup_{\bx \in \reals^d}|\cn(\bx)-N'_\bu(\snorm{\bx})| \leq \frac{\epsilon}{3}~.
\]
Therefore, we have
\[
\snorm{\cn - h'}_{L_2(\mu)}
\leq \snorm{\cn - h'}_{\infty}
= \sup_{\bx \in \reals^d}|\cn(\bx)-N'_\bu(\snorm{\bx})|
\leq \frac{\epsilon}{3}~.
\]

Combining the above with Eq.~\ref{eq:radial triangle1} and~\ref{eq:radial triangle2}, we have
\[
\snorm{\cn-f}_{L_2(\mu)}
\leq \snorm{\cn-h'}_{L_2(\mu)} + \snorm{h'-h}_{L_2(\mu)} + \snorm{h-f}_{L_2(\mu)}
\leq \frac{\epsilon}{3} + \frac{\epsilon}{3} + \frac{\epsilon}{3}
= \epsilon~.
\]

\subsection{Proof of Theorem~\ref{thm:1d function}}

\begin{lemma}
\label{lemma:dimension 1 networks}
Let $f:\reals \rightarrow \reals$ be a function that can be implemented by a neural network of width $n$ and constant depth. Then, $f$ can be implemented by a network of width $\poly(n)$ and depth $2$.
\end{lemma}
\begin{proof}
A neural network with input dimension $1$, constant depth, and width $n$, is piecewise linear with at most $\poly(n)$ pieces (\cite{telgarsky2015representation}). Therefore, $f$ consists of $m=\poly(n)$ linear pieces.

Let $-\infty = a_0 < a_1 < \ldots < a_{m-1} < a_m = \infty$ be such that $f$ is linear in every interval $(a_i,a_{i+1})$. For every $i \in [m]$ Let $\alpha_i$ be the derivative of $f$ in the linear interval $(a_{i-1},a_i)$.
Now, we have
\[
f(t) = f(a_1) - \alpha_1[a_1-t]_+ +  \sum_{2 \leq i \leq m-1} \left( \alpha_i[t-a_{i-1}]_+ - \alpha_i[t-a_i]_+ \right) + \alpha_m[t-a_{m-1}]_+~.
\]

Note that $f$ can be implemented by a network of depth $2$ and width $\poly(n)$.
In order to avoid bias in the output neuron, we implement the additive constant term $f(a_1)$ by adding a hidden neuron with fan-in $0$ and bias $1$, and connecting it to the output neuron with weight $f(a_1)$.
\end{proof}

Let $\epsilon = \frac{1}{\poly(d)}$.
Let $N:\reals^d \rightarrow \reals$ be a neural network of a constant depth and $\poly(d)$ width, such that $\E_{\bx \sim \mu}(N(\bx)-f(\bx))^2 \leq \left(\frac{\epsilon}{d}\right)^2$.
For $\bz \in \reals^{d-1}$ and $t \in \reals$ we denote $\bz_{i,t}=(z_1,\ldots,z_{i-1},t,z_i,\ldots,z_{d-1}) \in \reals^d$.
Since $\mu$ is such that the components are drawn independently, then for every $i \in [d]$ we have
\[
\E_{\bx \sim \mu}(N(\bx)-f(\bx))^2
= \E_{\bz \sim \mu_{[d]\setminus i}} \E_{t \sim \mu_i}(N(\bz_{i,t})-f(\bz_{i,t}))^2
\leq \left(\frac{\epsilon}{d}\right)^2~,
\]
and therefore for every $i$ there exists $\by \in \reals^{d-1}$ such that
\[
\E_{t \sim \mu_i}(N(\by_{i,t})-f(\by_{i,t}))^2 \leq \left(\frac{\epsilon}{d}\right)^2~.
\]
Let $f'_i:\reals \rightarrow \reals$ such that
\[
f'_i(t) = N(\by_{i,t}) - \sum_{j \in [d] \setminus \{i\}} f_j((\by_{i,t})_j)~.
\]
Note that
\begin{align}
\label{eq:f'i and fi}
\E_{t \sim \mu_i} \left(f'_i(t)-f_i(t)\right)^2
&= \E_{t \sim \mu_i} \left(N(\by_{i,t}) - \sum_{j \in [d] \setminus \{i\}} f_j((\by_{i,t})_j) - f_i(t)\right)^2 \nonumber
\\
&= \E_{t \sim \mu_i} \left(N(\by_{i,t}) - f(\by_{i,t})\right)^2
\leq \left(\frac{\epsilon}{d}\right)^2~.
\end{align}

Now, the function $f'_i$ can be implemented by a neural network of depth $2$ and width $\poly(d)$ as follows.
First, note the by Lemma~\ref{lemma:dimension 1 networks} it is sufficient to show that $f'_i$ can be implemented by a network $N'_i$ of a constant depth and $\poly(d)$ width.
Since $N$ is a network of constant depth, $\by$ is a constant, and $f_j((\by_{i,t})_j)$ for $j \in [d] \setminus \{i\}$ are also constants, implementing such $N'_i$ is straightforward.

Let $N'$ be the depth-$2$, width-$\poly(d)$ network such that $N'(\bx)=\sum_{i \in [d]}f'_i(x_i)$. This network is obtained from the networks for $f'_i$.
For every $i \in [d]$ let $g_i:\reals^d \rightarrow \reals$ be such that $g_i(\bx) = f_i(x_i)$. Also, let $g'_i:\reals^d \rightarrow \reals$ be such that $g'_i(\bx) = f'_i(x_i)$.
Note that $f(\bx) = \sum_{i \in [d]}g_i(\bx)$ and $N'(\bx) = \sum_{i \in [d]}g'_i(\bx)$.
Now, by Eq.~\ref{eq:f'i and fi}, for every $i \in [d]$ we have
\[
\E_{\bx \sim \mu}\left(g'_i(\bx)-g_i(\bx)\right)^2
= \E_{t \sim \mu_i}\left(f'_i(t) - f_i(t)\right)^2
\leq \left(\frac{\epsilon}{d}\right)^2~.
\]
Therefore, $\snorm{g'_i-g_i}_{L_2(\mu)} \leq \frac{\epsilon}{d}$.

Hence, we have
\[
\snorm{N'-f}_{L_2(\mu)}
= \norm{\sum_{i \in [d]}g'_i - \sum_{i \in [d]}g_i}_{L_2(\mu)}
\leq \sum_{i \in [d]} \snorm{g'_i - g_i}_{L_2(\mu)}
\leq d \cdot \frac{\epsilon}{d}
= \epsilon~.
\]

\subsection*{Acknowledgements}
This research is supported in part by European Research Council (ERC) grant 754705.

\bibliographystyle{abbrvnat}
\bibliography{bib}

\appendix

\section{Almost-bounded conditional density}
\label{app:almost-bounded conditional density}

In this section we show for some common distributions that they indeed have almost-bounded conditional densities.

\subsection{Gaussians, mixtures of Gaussians and Gaussian smoothing}

We use the following property of conditional normal distributions.

\begin{lemma} (e.g., \cite{bishop2006pattern})
\label{lemma:conditional normal}
Let $\cn(\mu,\Sigma)$ be a multivariate normal distribution on $\reals^d$.
For $\bx \in \reals^d$ we partition $\bx$ such that $\bx = (x_a,x_b)$, where $x_a \in \reals^q$ and $x_b \in \reals^{d-q}$.
Accordingly, we also partition $\mu = (\mu_a,\mu_b)$ and
\[
\Sigma = \begin{bmatrix}
        \Sigma_{aa} & \Sigma_{ab} \\
        \Sigma_{ba} & \Sigma_{bb}
        \end{bmatrix}~,
\]
where the dimensions of the mean vectors and the covariance matrix sub-blocks are chosen to match the sizes of $x_a,x_b$.
Let $\Lambda = \Sigma^{-1}$. We denote its partition that correspond to the partition of $\bx$ by
\[
\Lambda = \begin{bmatrix}
        \Lambda_{aa} & \Lambda_{ab} \\
        \Lambda_{ba} & \Lambda_{bb}
        \end{bmatrix}~.
\]
Then, the distribution of $x_a$ conditional on $x_b=\bc$ is the normal distribution $\cn(\bar{\mu},\bar{\Sigma})$, where
\[
\bar{\mu}
= \mu_a - \Lambda_{aa}^{-1}\Lambda_{ab}(\bc - \mu_b)
= \mu_a + \Sigma_{ab}\Sigma_{bb}^{-1}(\bc-\mu_b)~,
\]
and
\[
\bar{\Sigma}
= \Lambda_{aa}^{-1}
= \Sigma_{aa}-\Sigma_{ab}\Sigma_{bb}^{-1}\Sigma_{ba}~.
\]
\end{lemma}

\begin{proposition}
\label{prop:gaussian}
Let $\delta = \frac{1}{\poly(d)}$.
Let $\Sigma$ be a positive definite matrix of size $d \times d$ whose minimal eigenvalue is at least $\delta$.
Let $\mu \in \reals^d$.
Then, the multivariate normal distribution $\cn(\mu,\Sigma)$ has an almost-bounded conditional density.
\end{proposition}
\begin{proof}
Let $\Lambda = \Sigma^{-1}$.
Let $\lambda_1,\ldots,\lambda_d$ be the eigenvalues of $\Sigma$.
The eigenvalues of $\Lambda$ are $\lambda_1^{-1},\ldots,\lambda_d^{-1}$ and are at most $M=\frac{1}{\delta}$. Thus, $\text{trace}(\Lambda) = \sum_{i \in [d]}\lambda_i^{-1} \leq dM$.
Since $\Lambda$ is positive definite then all entries on its diagonal are positive, and since their sum is bounded by $dM$, then we have $0 < \Lambda_{ii} \leq dM$ for every $i \in [d]$.

Let $\bx \sim \cn(\mu,\Sigma)$, let $\bc \in \reals^{d-1}$, and let $i \in [d]$. We now consider the conditional distribution $x_i \; | \; x_1,\ldots,x_{i-1},x_{i+1},\ldots,x_d = \bc$.
This conditional distribution corresponds to Lemma~\ref{lemma:conditional normal} with $q=1$. Namely, this is a univariate normal distribution with variance $\Lambda_{aa}^{-1}$ where $\Lambda_{aa} \in \reals$. Since all entries on the diagonal of $\Lambda$ are bounded by $dM$, then the variance $\sigma^2$ of the conditional distribution satisfies $\sigma^2 \geq (dM)^{-1}$.
Since the density of a univariate normal distribution with variance $\sigma^2$ is bounded by $\frac{1}{\sqrt{2 \pi \sigma^2}}$, then the density of the conditional distribution is at most
$\frac{1}{\sqrt{2 \pi \sigma^2}} \leq \sqrt{\frac{dM}{2 \pi}} = \poly(d)$.
\end{proof}

We now consider Gaussian mixtures.

\begin{proposition}
Let $\Sigma_1,\ldots,\Sigma_k$ be positive definite matrices with eigenvalues at least $\delta = \frac{1}{\poly(d)}$.
Let $\mu_1,\ldots,\mu_k$ be vectors in $\reals^d$.
For $j \in [k]$ let $f^j$ be the density function of the normal distribution $\cn(\mu_j,\Sigma_j)$. Let $f$ be a density function such that $f(\bx) = \sum_{j \in [k]}w_j f^j(\bx)$ with $\sum_{j \in [k]}w_j = 1$.
Then, $f$ has an almost-bounded conditional density.
\end{proposition}
\begin{proof}
Let $i \in [d]$ and let $\bc \in \reals^{d-1}$. For $t \in \reals$ we denote $\bc_{i,t} = (c_1,\ldots,c_{i-1},t,c_i,\ldots,c_{d-1}) \in \reals^d$.
As we showed in the proof of Proposition~\ref{prop:gaussian}, there is $M = \poly(d)$ (that depends on $\delta$) such that for every $j \in [k]$ we have
\[
f^j_{i|[d] \setminus i}(t | \bc)
= \frac{f^j(\bc_{i,t})}{\int_{\reals}f^j(\bc_{i,t})dt}
\leq M~.
\]
Hence, we have
\[
f_{i|[d] \setminus i}(t | \bc)
= \frac{\sum_{j \in [k]} w_j f^j(\bc_{i,t})}{\int_{\reals}\sum_{j \in [k]} w_jf^j(\bc_{i,t})dt}
\leq \frac{\sum_{j \in [k]} w_j M \int_{\reals}f^j(\bc_{i,t})dt}{\sum_{j \in [k]} w_j\int_{\reals}f^j(\bc_{i,t})dt}
= M~.
\]
\end{proof}

Likewise, we show that the density obtained by applying Gaussian smoothing to a density function, has an almost-bounded conditional density.

\begin{proposition}
\label{prop:smoothing}
Let $\delta = \frac{1}{\poly(d)}$, and let $\Sigma$ be a positive definite matrix of size $d \times d$ whose minimal eigenvalue is at least $\delta$.
Let $g$ be the density function of the multivariate normal distribution $\cn(\zero,\Sigma)$.
Let $f$ be a density function and let $f' = f \star g$ be the convolution of $f$ and $g$. That is, $f'$ is the density function obtained from $f$ by Gaussian smoothing.
Then, $f'$ has an almost-bounded conditional density.
\end{proposition}
\begin{proof}
Let $i \in [d]$ and let $\bc \in \reals^{d-1}$. For $t \in \reals$ we denote $\bc_{i,t} = (c_1,\ldots,c_{i-1},t,c_i,\ldots,c_{d-1}) \in \reals^d$.
For $\by \in \reals^d$, let $g^\by:\reals^d \rightarrow \reals$ be such that $g^\by(\bx) = g(\bx - \by)$. Note that $g^\by$ is the density of the normal distribution $\cn(\by,\Sigma)$.
By the proof of Proposition~\ref{prop:gaussian}, there is $M = \poly(d)$ (that depends on $\delta$) such that for every $\by$, and every $\bc,t$ and $i$, we have
\begin{equation}
\label{eq:convMbound}
\frac{g(\bc_{i,t}-\by)}{\int_{\reals}g(\bc_{i,t}-\by)dt}
= \frac{g^\by(\bc_{i,t})}{\int_{\reals}g^\by(\bc_{i,t})dt}
= g^\by_{i|[d] \setminus i}(t | \bc)
\leq M~.
\end{equation}

Recall that
\[
f'(\bc_{i,t})
= (f \star g) (\bc_{i,t})
= \int_{\reals^d} f(\by)g(\bc_{i,t} - \by) d\by~.
\]
Now, we have
\begin{align*}
f'_{i|[d] \setminus i}(t | \bc)
&= \frac{f'(\bc_{i,t})}{\int_{\reals}f'(\bc_{i,t})dt}
= \frac{\int_{\reals^d} f(\by)g(\bc_{i,t} - \by) d\by}{\int_{\reals}\left[\int_{\reals^d} f(\by)g(\bc_{i,t} - \by) d\by\right]dt}
\\
&\stackrel{(Eq.~\ref{eq:convMbound})}{\leq} \frac{\int_{\reals^d} f(\by)M \left[\int_{\reals}g(\bc_{i,t}-\by)dt\right] d\by}{\int_{\reals}\left[\int_{\reals^d} f(\by)g(\bc_{i,t} - \by) d\by\right]dt}
\\
&= \frac{M \int_{\reals^d} \int_{\reals}f(\by) g(\bc_{i,t}-\by)dt d\by}{\int_{\reals^d} \int_{\reals} f(\by)g(\bc_{i,t} - \by) dt d\by}
= M~.
\end{align*}
\end{proof}

\subsection{Uniform distribution on the ball}

In the cases of Gaussians, Gaussian mixtures, and Gaussian smoothing, we showed that the conditional density of $x_i | x_1,\ldots,x_{i-1},x_{i+1},\ldots,x_d = \bc$ is bounded for every $\bc \in \reals^{d-1}$.
Note that the definition of almost-bounded conditional density allows the conditional density to be greater than $M$ for some set of $\bc \in \reals^{d-1}$ with a small marginal probability.
In the case of the uniform distribution over a ball in $\reals^d$, we show that we cannot bound the conditional density for all $\bc \in \reals^{d-1}$, but we can bound it for a set in $\reals^{d-1}$ with large marginal probability, which is sufficient by the definition of almost-bounded conditional density.

Let $\mu$ be the uniform distribution over the ball of a constant radius $R$ in $\reals^d$.
Let $\bc \in \reals^{d-1}$ be such that $\sum_{j \in [d-1]}c_j^2 = R^2 - \frac{1}{2^d}$.
Let $i \in [d]$. For $t \in \reals$, let $\bc_{i,t} = (c_1,\ldots,c_{i-1},t,c_i,\ldots,c_{d-1}) \in \reals^d$.
Note that $\mu_{i|[d] \setminus i}(t|\bc)=0$ for every $t$ such that $\sum_{j \in [d-1]}c_j^2 + t^2 > R^2$, namely, for every
\[
|t| > \sqrt{ R^2 - \sum_{j \in [d-1]}c_j^2 }
= \sqrt{ \frac{1}{2^d} }
= \frac{1}{2^{d/2}}~.
\]
Hence, the conditional density $\mu_{i|[d] \setminus i}(t|\bc) = \frac{\mu(\bc_{i,t})}{\mu_{[d] \setminus i}(\bc)}$ is uniform on the interval $\left[-\frac{1}{2^{d/2}},\frac{1}{2^{d/2}}\right]$.
Therefore, we have
\[
\mu_{i|[d] \setminus i}(t|\bc)
= \frac{1}{2 \cdot \frac{1}{2^{d/2}}}
= 2^{\frac{d}{2} - 1}~.
\]
Thus, for such $\bc$ we cannot bound $\mu_{i|[d] \setminus i}(t|\bc)$ with a polynomial.
However, as we show in the following proposition, the marginal probability to obtain such $\bc$ is small, and $\mu$ has an almost-bounded conditional density.

\begin{proposition}
Let $\mu$ be the uniform distribution over the ball of radius $R \geq \frac{1}{\poly(d)}$ in $\reals^d$. Then, $\mu$ has an almost-bounded conditional density.
\end{proposition}
\begin{proof}
Let $\epsilon = \frac{1}{\poly(d)}$ and let $M = \frac{\sqrt{d}}{2R \sqrt{2\epsilon}}$. Let $i \in [d]$ and let $\bc \in \reals^{d-1}$.
We denote $r = \sqrt{\sum_{j \in [d-1]}c_j^2}$.
Note that $\mu_{i|[d] \setminus i}(t|\bc)$ is the uniform distribution over the interval $\left[-\sqrt{R^2-r^2}, \sqrt{R^2-r^2}\right]$.
Hence, for every $t$ in this interval we have
\[
\mu_{i|[d] \setminus i}(t|\bc)
= \frac{1}{2\sqrt{R^2-r^2}}~.
\]
Note that if $r^2 \leq R^2-\frac{1}{4M^2}$ then $\mu_{i|[d] \setminus i}(t|\bc) \leq M$.
Therefore, we have
\begin{align*}
Pr_{\bc \sim \mu_{[d]\setminus i}}\left(\exists t \text{\ s.t.\ }\mu_{i|[d] \setminus i}(t|\bc) > M\right)
&\leq Pr_{\bc \sim \mu_{[d]\setminus i}}\left(\sum_{j \in [d-1]}c_j^2 > R^2-\frac{1}{4M^2} \right) \nonumber
\\
&\leq Pr_{\bx \sim \mu}\left(\sum_{j \in [d]}x_j^2 > R^2-\frac{1}{4M^2} \right)~.
\end{align*}

Let $V_d(R)$ be the volume of the ball of radius $R$ in $\reals^d$. Recall that $V_d(R)=V_d(1) \cdot R^d$.
Note that the above equals to
\begin{align*}
\frac{1}{V_d(R)} \cdot \left(V_d(R) - V_d\left(\sqrt{R^2-\frac{1}{4M^2}}\right)\right)
&= 1 - \frac{\left(\sqrt{R^2-\frac{1}{4M^2}}\right)^{d}}{R^d} \nonumber
\\
&= 1 - \left(1-\frac{1}{4M^2R^2}\right)^{d/2}~.
\end{align*}

By Bernoulli's inequality, for every $z \geq -1$ and $y \geq 1$ we have $(1+z)^y \geq 1 + yz$. Therefore, the above is at most
\[
1 - \left(1-\frac{d}{8M^2R^2}\right)
= \frac{d}{8M^2R^2}~.
\]

Plugging in $M = \frac{\sqrt{d}}{2R \sqrt{2\epsilon}}$, we obtain
\[
Pr_{\bc \sim \mu_{[d]\setminus i}}\left(\exists t \text{\ s.t.\ }\mu_{i|[d] \setminus i}(t|\bc) > M\right)
\leq \epsilon~.
\]
\end{proof}

\subsection{Distributions from existing depth-separation results}

As we described in Section~\ref{sec:intro}, the depth-separation result of \cite{telgarsky2016benefits}, and the results that rely on it (e.g., \cite{safran2017depth,yarotsky2017error,liang2016deep}), are with respect to the uniform distribution on $[0,1]^d$. Thus, each component is chosen i.i.d. from the uniform distribution on the interval $[0,1]$, and therefore its conditional density is bounded by the constant $1$.

The depth-separation result of \cite{daniely2017depth} is for the function $f(\bx_1,\bx_2) = \sin(\pi d^3 \inner{\bx_1,\bx_2})$ with respect to the uniform distribution on $\bbs^{d-1} \times \bbs^{d-1}$, namely, both $\bx_1$ and $\bx_2$ are on the unit sphere. In \cite{safran2019depth}, it is shown that this result can be easily reduced to a depth-separation result for the function $f(\bx)=\sin(\frac{1}{2} \pi d^3 \norm{\bx})$ and an $L_\infty$-type approximation. Moreover, from their proof it follows that this reduction applies also to an $L_2$ approximation with respect to an input $\bx = \frac{\bx_1 + \bx_2}{2}$ where $\bx_1$ and $\bx_2$ are drawn i.i.d. from the uniform distribution on $\bbs^{d-1}$.
We now show that this distribution has an almost-bounded conditional density.
We first find the density function of $\norm{\bx}$.

\begin{lemma}
\label{lemma:marginal daniely}
Let $\bx = \frac{\bx_1 + \bx_2}{2}$ where $\bx_1$ and $\bx_2$ are drawn i.i.d. from the uniform distribution on $\bbs^{d-1}$. Then, the distribution of $\norm{\bx}$ has the density
\[
f_r(r) = \frac{1}{B\left(\frac{1}{2},\frac{d-1}{2}\right)} 2^{d-1} r^{d-2} (1-r^2)^{\frac{d-3}{2}}~,
\]
where $B(\alpha,\beta) = \frac{\Gamma(\alpha)\Gamma(\beta)}{\Gamma(\alpha+\beta)}$ is the beta function, and $r \in (0,1)$.
\end{lemma}
\begin{proof}
Let $\bx = \frac{\bx_1 + \bx_2}{2}$ where $\bx_1$ and $\bx_2$ are drawn i.i.d. from the uniform distribution on $\bbs^{d-1}$. Note that
\begin{equation}
\label{eq:square norm}
\norm{\bx}^2
= \frac{1}{4} \left(\norm{\bx_1}^2 + \norm{\bx_2}^2 + 2 \bx_1^\top\bx_2\right)
= \frac{1}{4} \left(2 + 2 \bx_1^\top\bx_2\right)
= \frac{1}{2} \left(1 + \bx_1^\top\bx_2\right)~.
\end{equation}

Since $\bx_1$ and $\bx_2$ are independent and uniformly distributed on the sphere, then the distribution of $\bx_1^\top\bx_2$ equals to the distribution of $(1,0,\ldots,0) \bx_2$, which equals to the
marginal distribution of the first component of $\bx_2$. Let $z$ be the first component of $\bx_2$. By standard results (cf. \cite{fang2018symmetric}), the distribution of $z^2$ is $\betadist(\frac{1}{2},\frac{d-1}{2})$, namely, a Beta distribution with parameters $\frac{1}{2},\frac{d-1}{2}$. Thus, the density of $z^2$ is
\[
f_{z^2}(y) = \frac{1}{B\left(\frac{1}{2},\frac{d-1}{2}\right)} y^{-\frac{1}{2}} (1-y)^{\frac{d-3}{2}}~,
\]
where $B(\alpha,\beta) = \frac{\Gamma(\alpha)\Gamma(\beta)}{\Gamma(\alpha+\beta)}$ is the beta function, and $y \in (0,1)$.

Performing a variable change, we obtain the density of $|z|$, which equals to the density of $|\bx_1^\top\bx_2|$.
\[
f_{|\bx_1^\top\bx_2|}(y) = f_{|z|}(y)
= f_{z^2}(y^2) \cdot 2y
= \frac{1}{B\left(\frac{1}{2},\frac{d-1}{2}\right)} y^{-1} (1-y^2)^{\frac{d-3}{2}} \cdot 2y
= \frac{2}{B\left(\frac{1}{2},\frac{d-1}{2}\right)} (1-y^2)^{\frac{d-3}{2}}~,
\]
where $y \in (0,1)$.
Let $f_{\bx_1^\top\bx_2}$ be the density of $\bx_1^\top\bx_2$. Note that for every $y \in (-1,1)$ we have $f_{\bx_1^\top\bx_2}(y) = f_{\bx_1^\top\bx_2}(-y)$. Hence, for every $y \in (-1,1)$,
\[
f_{\bx_1^\top\bx_2}(y) =
\frac{1}{2} f_{|\bx_1^\top\bx_2|}(|y|)
= \frac{1}{B\left(\frac{1}{2},\frac{d-1}{2}\right)} (1-y^2)^{\frac{d-3}{2}}~.
\]

Performing a variable change again, we obtain the density of  $\frac{1}{\sqrt{2}} \cdot \sqrt{1+\bx_1^\top\bx_2}$.
\begin{align*}
f_{\frac{1}{\sqrt{2}} \cdot \sqrt{1+\bx_1^\top\bx_2}}(y)
&= f_{\bx_1^\top\bx_2}(2y^2-1) \cdot 4y
= \frac{1}{B\left(\frac{1}{2},\frac{d-1}{2}\right)} (1-(4y^4-4y^2+1))^{\frac{d-3}{2}} \cdot 4y
\\
&= \frac{1}{B\left(\frac{1}{2},\frac{d-1}{2}\right)} (2y)^{d-3} (1-y^2)^{\frac{d-3}{2}} \cdot 4y
= \frac{1}{B\left(\frac{1}{2},\frac{d-1}{2}\right)} 2^{d-1} y^{d-2} (1-y^2)^{\frac{d-3}{2}}~.
\end{align*}

\stam{
Performing variable changes again, we obtain the density of $1+\bx_1^\top\bx_2$:
\[
f_{1+\bx_1^\top\bx_2}(y) = f_{\bx_1^\top\bx_2}(y-1)
= \frac{1}{B\left(\frac{1}{2},\frac{d-1}{2}\right)} (1-(y^2-2y+1))^{\frac{d-3}{2}}
= \frac{1}{B\left(\frac{1}{2},\frac{d-1}{2}\right)} y^{\frac{d-3}{2}} (2-y)^{\frac{d-3}{2}}~,
\]
the density of $\sqrt{1+\bx_1^\top\bx_2}$:
\[
f_{\sqrt{1+\bx_1^\top\bx_2}}(y) = f_{1+\bx_1^\top\bx_2}(y^2) \cdot 2y
= \frac{1}{B\left(\frac{1}{2},\frac{d-1}{2}\right)} y^{d-3} (2-y^2)^{\frac{d-3}{2}} 2y
= \frac{2}{B\left(\frac{1}{2},\frac{d-1}{2}\right)} y^{d-2} (2-y^2)^{\frac{d-3}{2}}~,
\]
and the density of $\frac{1}{\sqrt{2}} \cdot \sqrt{1+\bx_1^\top\bx_2}$:
\begin{align*}
f_{\frac{1}{\sqrt{2}} \cdot \sqrt{1+\bx_1^\top\bx_2}}(y)
&= f_{\sqrt{1+\bx_1^\top\bx_2}}(\sqrt{2}y) \cdot \sqrt{2}
= \frac{2}{B\left(\frac{1}{2},\frac{d-1}{2}\right)} 2^\frac{d-2}{2} y^{d-2} (2-2y^2)^{\frac{d-3}{2}} \sqrt{2}
\\
&= \frac{1}{B\left(\frac{1}{2},\frac{d-1}{2}\right)} 2^{d-1} y^{d-2} (1-y^2)^{\frac{d-3}{2}}~.
\end{align*}
}%

Note that by Eq.~\ref{eq:square norm} we have
\[
\norm{\bx} = \sqrt{\frac{1 + \bx_1^\top\bx_2}{2}}~,
\]
and therefore the density of $\norm{\bx}$ is
\[
f_r(r) = f_{\frac{1}{\sqrt{2}} \cdot \sqrt{1+\bx_1^\top\bx_2}}(r)
= \frac{1}{B\left(\frac{1}{2},\frac{d-1}{2}\right)} 2^{d-1} r^{d-2} (1-r^2)^{\frac{d-3}{2}}~.
\]
\end{proof}

\begin{proposition}
Let $\bx = \frac{\bx_1 + \bx_2}{2}$ where $\bx_1$ and $\bx_2$ are drawn i.i.d. from the uniform distribution on $\bbs^{d-1}$. Then the distribution of $\bx$ has an almost-bounded conditional density.
\end{proposition}
\begin{proof}
Let $\epsilon = \frac{1}{\poly(d)}$.
Let $f_r$ be the distribution of $\norm{\bx}$. By Lemma~\ref{lemma:marginal daniely}, we have
\begin{equation}
\label{eq:defining f_r}
f_r(r) = \frac{1}{B\left(\frac{1}{2},\frac{d-1}{2}\right)} 2^{d-1} r^{d-2} (1-r^2)^{\frac{d-3}{2}}~.
\end{equation}

Let $\mu:\reals^d \rightarrow \reals$ be the density function on $\reals^d$ that is induced by $f_r$. That is, $\bx \sim \mu$ has the same distribution as $r \bu$ where $r \sim f_r$ and $\bu$ is distributed uniformly on $\bbs^{d-1}$.
Let $i \in [d]$.
For simplicity, we always assume in this proof that $d \geq 5$ (note that the definition of almost-bounded conditional density is not sensitive to the behavior of the density for small values of $d$).

We will first find $\delta_1,\delta_2 \leq \frac{1}{\poly(d)}$ such that
\begin{equation}
\label{eq:c bounds}
Pr_{\bc \sim \mu_{[d] \setminus i}}\left(\delta_1 \leq \norm{\bc} \leq 1-\delta_2\right) \geq 1-\epsilon~.
\end{equation}
Then, we will show that there is $M = \poly(d)$ such that for every $\bc \in \reals^{d-1}$ with $\delta_1 \leq \norm{\bc} \leq 1-\delta_2$ and every $t \in (-1,1)$ we have
\begin{equation}
\label{eq:bounding the conditional density}
\mu_{i|[d] \setminus i}(t|\bc) \leq M~.
\end{equation}

We start with $\delta_2$. Note that
\begin{equation}
\label{eq:bound B}
B\left(\frac{1}{2},\frac{d-1}{2}\right)
= \frac{\Gamma(\frac{1}{2})\Gamma(\frac{d-1}{2})}{\Gamma(\frac{d}{2})}
\geq \frac{\Gamma(\frac{1}{2})\Gamma(\frac{d}{2}-1)}{\Gamma(\frac{d}{2})}
= \frac{\Gamma(\frac{1}{2})}{\frac{d}{2}-1}
\geq \frac{2\Gamma(\frac{1}{2})}{d}
= \frac{2\sqrt{\pi}}{d}
\geq \frac{1}{d}~.
\end{equation}
Let $\delta_2 = 1- \sqrt{1-\frac{\epsilon}{32d}}$. By the above and Eq.~\ref{eq:defining f_r}, for every $r \in (1-\delta_2,1)$ we have
\[
f_r(r)
\leq d 2^{d-1} r^{d-2} (1-r^2)^{\frac{d-3}{2}}
\leq d 2^{d-1} \left(1-(1-\delta_2)^2\right)^{\frac{d-3}{2}}
= d 2^{d-1} \left(\frac{\epsilon}{32d}\right)^{\frac{d-3}{2}}~.
\]
Hence,
\begin{align}
\label{eq:c upper bound}
Pr_{\bc \sim \mu_{[d] \setminus i}}\left(\norm{\bc} \geq 1-\delta_2 \right)
&\leq  Pr_{r \sim f_r}\left(r \geq 1-\delta_2 \right)
\leq d \cdot 2^{d-1} \left(\frac{\epsilon}{32d}\right)^{\frac{d-3}{2}} \cdot \delta_2 \nonumber
\\
&\leq d \cdot 4^{\frac{d-3}{2}} \cdot 4 \left(\frac{\epsilon}{32d}\right)^{\frac{d-3}{2}}
= 4d \left(\frac{\epsilon}{8d}\right)^{\frac{d-3}{2}}
\leq 4d \cdot \frac{\epsilon}{8d}
=\frac{\epsilon}{2}~.
\end{align}

We now turn to $\delta_1$.
By \cite{fang2018symmetric}, the marginal distribution $\bc \sim \mu_{[d] \setminus i}$ is such that $\norm{\bc} = r \alpha$, where $r$ and $\alpha$ are independent, $r \sim f_r$, and $\alpha^2 \sim \betadist\left(\frac{d-1}{2},\frac{1}{2}\right)$, namely, a Beta distribution with parameters $\frac{d-1}{2},\frac{1}{2}$.
Hence, we have
\begin{equation}
\label{eq:c lower bound}
Pr_{\bc \sim \mu_{[d] \setminus i}}\left(\norm{\bc} \leq \delta_1 \right)
\leq Pr_{r \sim f_r}\left(r \leq \sqrt{\delta_1}\right) + Pr_{\beta \sim \betadist\left(\frac{d-1}{2},\frac{1}{2}\right)}\left(\sqrt{\beta} \leq \sqrt{\delta_1}\right)~.
\end{equation}

We now bound the two part of the above right hand side.
For $\delta = \frac{\epsilon}{16d}$, we have by Eq.~\ref{eq:defining f_r} and~\ref{eq:bound B} that for every $r \in (0,\delta)$,
\[
f_r(r)
\leq d 2^{d-1} r^{d-2} (1-r^2)^{\frac{d-3}{2}}
\leq d 2^{d-1} \delta^{d-2}
= d 2^{d-1} \left(\frac{\epsilon}{16d}\right)^{d-2}
= 2d \left(\frac{\epsilon}{8d}\right)^{d-2}
\leq 2d \cdot \frac{\epsilon}{8d}
= \frac{\epsilon}{4}~.
\]
Thus, for $\delta_1 = \delta^2$ we have
\begin{equation}
\label{eq:bound f_r for delta1}
Pr_{r \sim f_r}\left(r \leq \sqrt{\delta_1} \right)
= Pr_{r \sim f_r}\left(r \leq \delta \right)
\leq \delta \cdot \frac{\epsilon}{4}
\leq \frac{\epsilon}{4}~.
\end{equation}

Moreover, we have
\begin{align*}
Pr_{\beta \sim \betadist\left(\frac{d-1}{2},\frac{1}{2}\right)}\left(\sqrt{\beta} \leq \sqrt{\delta_1}\right)
&= Pr_{\beta \sim \betadist\left(\frac{d-1}{2},\frac{1}{2}\right)}\left(\beta \leq \delta_1\right)
\\
&= \int_{0}^{\delta_1}\frac{1}{B\left(\frac{d-1}{2},\frac{1}{2}\right)}\beta^{\frac{d-1}{2}-1}(1-\beta)^{\frac{1}{2}-1} d\beta
\\
&\leq \delta_1 \cdot \frac{1}{B\left(\frac{d-1}{2},\frac{1}{2}\right)} \cdot \delta_1^{\frac{d-3}{2}} \cdot \frac{1}{\sqrt{1-\delta_1}}
\end{align*}
Since $0 < \delta_1 \leq \frac{1}{2}$, and by plugging in Eq.~\ref{eq:bound B}, the above is at most
\begin{equation}
\label{eq:bound beta for delta1}
d \cdot \delta_1^{\frac{d-1}{2}} \cdot \frac{1}{\sqrt{1-\frac{1}{2}}}
\leq d \cdot \delta_1^{\frac{1}{2}} \cdot \sqrt{2}
= \sqrt{2}d \cdot \frac{\epsilon}{16d}
\leq \frac{\epsilon}{4}~.
\end{equation}

Combining Eq.~\ref{eq:c lower bound},~\ref{eq:bound f_r for delta1} and~\ref{eq:bound beta for delta1}, we have
\[
Pr_{\bc \sim \mu_{[d] \setminus i}}\left(\norm{\bc} \leq \delta_1 \right) \leq \frac{\epsilon}{2}~.
\]
Then, Eq.~\ref{eq:c bounds} follows by combining the above with Eq.~\ref{eq:c upper bound}.
Thus, it remains to show that there is $M = \poly(d)$ such that for every $\bc \in \reals^{d-1}$ with $\delta_1 \leq \norm{\bc} \leq 1-\delta_2$ and every $t \in (-1,1)$, Eq.~\ref{eq:bounding the conditional density} holds.

Let $A_d$ be the surface area of the unit sphere in $\reals^d$. Note that for every $\bx \neq \zero$ in the unit ball, we have
\begin{align*}
\mu(\bx)
&= \frac{f_r(\norm{\bx})}{\norm{\bx}^{d-1} A_d}
= \frac{1}{B\left(\frac{1}{2},\frac{d-1}{2}\right)} 2^{d-1} \norm{\bx}^{d-2} (1-\norm{\bx}^2)^{\frac{d-3}{2}} \cdot \frac{1}{\norm{\bx}^{d-1} A_d}
\\
&= \frac{1}{A_d B\left(\frac{1}{2},\frac{d-1}{2}\right)} 2^{d-1} (1-\norm{\bx}^2)^{\frac{d-3}{2}} \cdot \frac{1}{\norm{\bx}}~.
\end{align*}
For $t \in \reals$, we denote $\bc_{i,t} = (c_1,\ldots,c_{i-1},t,c_i,\ldots,c_{d-1}) \in \reals^d$.
Now, we have
\begin{align*}
\mu_{[d] \setminus i}(\bc)
= \int_{-1}^{1} \mu(\bc_{i,t}) dt
= \int_{-\sqrt{1-\norm{\bc}^2}}^{\sqrt{1-\norm{\bc}^2}} \frac{1}{A_d B\left(\frac{1}{2},\frac{d-1}{2}\right)} 2^{d-1} (1-(\norm{\bc}^2+t^2))^{\frac{d-3}{2}} \cdot \frac{1}{\sqrt{\norm{\bc}^2+t^2}} dt~.
\end{align*}
Performing the variable change $z = \sqrt{\norm{\bc}^2+t^2}$, the above equals
\begin{align*}
2 \int_{\norm{\bc}}^{1} &\frac{1}{A_d B\left(\frac{1}{2},\frac{d-1}{2}\right)} 2^{d-1} (1-z^2)^{\frac{d-3}{2}} \cdot \frac{1}{z} \cdot \frac{z}{\sqrt{z^2-\norm{\bc}^2}} dz
\\
&\geq 2 \int_{\norm{\bc}}^{1} \frac{1}{A_d B\left(\frac{1}{2},\frac{d-1}{2}\right)} 2^{d-1} (1-z^2)^{\frac{d-3}{2}} \cdot \frac{1}{z} dz
\\
&= \frac{2^{d}}{A_d B\left(\frac{1}{2},\frac{d-1}{2}\right)} \int_{\norm{\bc}}^{1} (1+z)^{\frac{d-3}{2}}(1-z)^{\frac{d-3}{2}} \cdot \frac{1}{z} dz
\\
&\geq \frac{2^{d}}{A_d B\left(\frac{1}{2},\frac{d-1}{2}\right)} (1+\norm{\bc})^{\frac{d-3}{2}} \int_{\norm{\bc}}^{1} (1-z)^{\frac{d-3}{2}} dz~.
\end{align*}
By plugging in
\[
\int_{\norm{\bc}}^{1} (1-z)^{\frac{d-3}{2}} dz
= \left. -\frac{(1-z)^{\frac{d-3}{2}+1}}{\frac{d-3}{2}+1} \right|_{\norm{\bc}}^1
= \frac{2(1-\norm{\bc})^{\frac{d-1}{2}}}{d-1}~,
\]
we get
\[
\frac{2^{d+1}(1+\norm{\bc})^{\frac{d-3}{2}}(1-\norm{\bc})^{\frac{d-1}{2}}}{A_d B\left(\frac{1}{2},\frac{d-1}{2}\right)(d-1)}~.
\]
Hence,
\begin{align*}
\mu_{i|[d] \setminus i}(t|\bc)
&= \frac{\mu(\bc_{i,t})}{\mu_{[d] \setminus i}(\bc)}
\\
&\leq \frac{1}{A_d B\left(\frac{1}{2},\frac{d-1}{2}\right)} 2^{d-1} (1-\norm{\bc_{i,t}}^2)^{\frac{d-3}{2}} \cdot \frac{1}{\norm{\bc_{i,t}}} \cdot \frac{A_d B\left(\frac{1}{2},\frac{d-1}{2}\right) (d-1)}{2^{d+1}(1+\norm{\bc})^{\frac{d-3}{2}}(1-\norm{\bc})^{\frac{d-1}{2}}}
\\
&= (1-\norm{\bc_{i,t}}^2)^{\frac{d-3}{2}} \cdot \frac{1}{\norm{\bc_{i,t}}} \cdot \frac{d-1}{2^{2} (1+\norm{\bc})^{\frac{d-3}{2}} (1-\norm{\bc})^{\frac{d-1}{2}}}
\\
&\leq (1+\norm{\bc})^{\frac{d-3}{2}}(1-\norm{\bc})^{\frac{d-3}{2}} \cdot \frac{1}{\norm{\bc}} \cdot  \frac{d-1}{4 (1+\norm{\bc})^{\frac{d-3}{2}} (1-\norm{\bc})^{\frac{d-1}{2}}}
\\
&= \frac{1}{\norm{\bc}} \cdot \frac{d-1}{4 (1-\norm{\bc})}~.
\end{align*}
Now, since $\delta_1 \leq \norm{\bc} \leq 1-\delta_2$, the above is at most
\[
\frac{1}{\delta_1} \cdot \frac{d-1}{4 \delta_2} \leq \poly(d)~.
\]
\end{proof}

\cite{eldan2016power} showed separation between depth $2$ and $3$ for a $\poly(d)$-Lipschitz radial function $f:\reals^d \rightarrow \reals$ with respect to a distribution with density
\[
\mu(\bx) = \left(\frac{R_d}{\norm{\bx}}\right)^d J_{d/2}^2(2 \pi R_d \norm{\bx})~,
\]
where $R_d$ is the radius of the unit-volume Euclidean ball in $\reals^d$, and $J_{d/2}$ is a Bessel function of the first kind. An analysis of its conditional density requires some investigation of Bessel functions and is not included here.
However, it is not hard to show that for every polynomial $p(d)$, there is a distribution $\mu'$ (obtained by applying Gaussian smoothing to $\mu$ and has an almost-bounded conditional density by Proposition~\ref{prop:smoothing}), such that the function $f$ can be expressed by a depth-$3$ network but cannot be approximated by a depth-$2$ network with a Lipschitz constant bounded by $p(d)$. This follows from the fact that if there was a Lipschitz approximating network under $\mu'$, it would also be approximating under the slightly different distribution $\mu$.

\end{document}